\documentclass{article}
\usepackage{kr}

\usepackage{times}
\usepackage{soul}
\usepackage{url}
\usepackage[hidelinks]{hyperref}
\usepackage[utf8]{inputenc}
\usepackage[small]{caption}
\usepackage{graphicx}
\usepackage{amsmath}
\usepackage{amsthm}
\usepackage{booktabs}
\usepackage{algorithm}
\usepackage{algorithmic}
\usepackage{stmaryrd}
\usepackage{amssymb}
\usepackage{amsthm}
\usepackage{mathtools}
\usepackage{nicefrac}
\usepackage{dsfont}
\usepackage{hyperref}
\usepackage{cleveref}
\usepackage{crossreftools}
\usepackage{cancel}
\usepackage{graphicx}
\usepackage{proof}
\usepackage{lipsum}
\usepackage[super]{nth}

\usepackage{tikz}
\usetikzlibrary{shapes.symbols}
\usepackage{tikz-cd}
\usepackage{tikzit} 
\usetikzlibrary{decorations.pathreplacing}
\usetikzlibrary{positioning}
\usetikzlibrary{patterns.meta}
\usetikzlibrary {arrows.meta}
\usepackage{tikz-qtree}
\usepackage{tikzducks}
\usepackage[inline,shortlabels]{enumitem}
\usepackage{wrapfig}
\usepackage[backgroundcolor=orange!20, textcolor={Dark Ruby Red}, textsize=tiny]{todonotes}
\usepackage{pdfcomment}
\usepackage{bbding} 
\newif\ifsubmission
\newif\ifappendix
\ifsubmission
  \usepackage[bibliography=common,appendix=strip]{apxproof} 
\else
  \usepackage[bibliography=common]{apxproof} 
\fi

\NewCommandCopy{\proofqedsymbol}{\qedsymbol}
\newcommand{\exampleqedsymbol}{{$\triangle$}}
\AtBeginEnvironment{proof}{\renewcommand{\qedsymbol}{\proofqedsymbol}}
\AtBeginEnvironment{nestedproof}{\renewcommand{\qedsymbol}{\nestedproofqedsymbol}}
\AtBeginEnvironment{example}{%
  \pushQED{\qed}\renewcommand{\qedsymbol}{\exampleqedsymbol}%
}
\AtEndEnvironment{example}{\popQED\endexample}

\usepackage[xcolor, hyperref, cleveref, notion, quotation, electronic]{knowledge}
\usepackage{mathcommand}
\knowledgeconfigure{quotation, protect quotation={tikzcd}}
\knowledgeconfigure{diagnose line=true, diagnose bar=true}

\definecolor{Dark Ruby Red}{HTML}{5d1416}
\definecolor{Dark Blue Sapphire}{HTML}{003c47} 
\definecolor{Dark Gamboge}{HTML}{be7c00}

\IfKnowledgePaperModeTF{
}{
    \knowledgestyle{intro notion}{color={Dark Ruby Red}, emphasize}
    \knowledgestyle{notion}{color={Dark Blue Sapphire}}
    \hypersetup{
        colorlinks=true,
        breaklinks=true,
        linkcolor={Dark Blue Sapphire}, 
        citecolor={Dark Blue Sapphire}, 
        filecolor={Dark Blue Sapphire}, 
        urlcolor={Dark Blue Sapphire},
    }
    \IfKnowledgeElectronicModeTF{
    }{
        \knowledgeconfigure{anchor point color={Dark Ruby Red}, anchor point shape=corner}
        \knowledgestyle{intro unknown}{color={Dark Gamboge}, emphasize}
        \knowledgestyle{intro unknown cont}{color={Dark Gamboge}, emphasize}
        \knowledgestyle{kl unknown}{color={Dark Gamboge}}
        \knowledgestyle{kl unknown cont}{color={Dark Gamboge}}
    }
}
 
\input{header/kl-abbreviations.kl} 
\input{header/kl-complexity.kl} 
\input{header/kl-knowledges.kl}
\knowledgenewrobustcmd{\dcup}{\mathbin{\cmdkl{\uplus}}} 
\knowledgenewrobustcmd{\bigdcup}{\mathop{\cmdkl{\biguplus}}} 

\newcommand{\set}[1]{\{#1\}}

\newcommand{\partsof}[1]{2^{#1}}
\newrobustcmd{\Nat}{\mathbb{N}}

\knowledgenewrobustcmd{\Sh}{\cmdkl{\mathrm{Sh}}} 

\knowledgenewrobustcmd{\omqsat}{\mathrel{\cmdkl{\models}}} 

\knowledgenewrobustcmd{\polyrx}{\mathrel{\cmdkl{\le_{\mathsf{P}}}}} 

\knowledgenewrobustcmd{\polyeq}{\mathrel{\cmdkl{\equiv_{\mathsf{P}}}}} 

\knowledgenewrobustcmd{\homto}[1][]{\mathrel{\cmdkl{\xrightarrow{\smash{\textit{\tiny #1 \!hom}}}}}} 

\knowledgenewrobustcmd{\dom}{\cmdkl{\mathrm{dom}}} 

\knowledgenewrobustcmd{\domain}{\cmdkl{\mathit{dom}}} 

\newcommand{\subendo}{{\textup{\textsf{n}}}}
\newcommand{\subexo}{{\textup{\textsf{x}}}}
\knowledgenewrobustcmd{\Dn}[1][\D]{#1_{\cmdkl{\subendo}}}
\knowledgenewrobustcmd{\Dx}[1][\D]{#1_{\cmdkl{\subexo}}}

\knowledgenewrobustcmd{\scorefun}[1][]{\cmdkl{\mathbf{\xi}_{#1}}} 

\definecolor{light-gray}{gray}{0.9}
\newcommand{\proofcase}[1]{\noindent\colorbox{light-gray}{#1}~~}

\newcommand{\Ra}{\Rightarrow}

\newcommand{\defeq}{\vcentcolon=}

\renewcommand{\le}{\leqslant}
\renewcommand{\ge}{\geqslant}

\newcommand{\ic}{\sqsubseteq} 

\newcommand{\A}{\mathcal{A}}

\newcommand{\C}{\mathcal{C}}

\newcommand{\D}{\mathcal{D}}

\newcommand{\I}{\mathcal{I}}

\newcommand{\J}{\mathcal{J}}

\newcommand{\K}{\mathcal{K}}

\renewcommand{\L}{\mathcal{L}}

\newcommand{\Q}{\mathcal{Q}}
\newcommand{\lQ}{\mathbb{Q}} 
\newcommand{\R}{\mathcal{R}}

\knowledgenewrobustcmd{\Sym}{\cmdkl{\mathfrak{S}}} 
\newcommand{\T}{\mathcal{T}}

\def\Amsf{A}

\def\rmsf{r}
\def\Gmsf{G}
\def\Hmsf{H}

\knowledgenewrobustcmd{\dnames}{\cmdkl{\mathsf{N_D}}}
\knowledgenewrobustcmd{\cnames}{\cmdkl{\mathsf{N_C}}}
\knowledgenewrobustcmd{\rnames}{\cmdkl{\mathsf{N_R}}}
\knowledgenewrobustcmd{\inames}{\cmdkl{\mathsf{N_I}}}
\knowledgenewrobustcmd{\nulls}{\cmdkl{\mathsf{N_A}}}
\knowledgenewrobustcmd{\vnames}{\cmdkl{\mathsf{N_V}}}
\knowledgenewrobustcmd{\irnames}{\cmdkl{\mathsf{N^{\pm}_R}}}
\knowledgenewrobustcmd{\NRpm}{\cmdkl{\mathsf{N}_R^{\pm}}}
\knowledgenewrobustcmd{\terms}{\cmdkl{\mathsf{terms}}}
\knowledgenewrobustcmd{\mods}{\cmdkl{\mathsf{Mod}}}
\knowledgenewrobustcmd{\canmod}[1]{\cmdkl{\I_{#1}}}

\knowledgenewrobustcmd{\mindl}{\cmdkl{\ensuremath{\mathcal{L}_{\mathsf{min}}}}}
\knowledgenewrobustcmd{\dllitec}{\cmdkl{\ensuremath{\text{DL-Lite}_{\mathsf{core}}}}}
\knowledgenewrobustcmd{\dlliter}{\cmdkl{\ensuremath{\text{DL-Lite}_{\mathcal{R}}}}}
\knowledgenewrobustcmd{\EL}{\cmdkl{\mathcal{EL}}}
\knowledgenewrobustcmd{\ELI}{\cmdkl{\mathcal{ELI}}}
\knowledgenewrobustcmd{\horndl}{\cmdkl{\mathcal{ELHIF}_{\bot}}}

\newenvironment{equation-inline}{\refstepcounter{equation}}{\hfill(\theequation)\\}

\knowledgenewrobustcmd{\atoms}{\cmdkl{\textit{atoms}}}
\knowledgenewrobustcmd{\vars}{\cmdkl{\textit{vars}}}
\knowledgenewrobustcmd{\const}{\cmdkl{\textit{ind}}}
\knowledgenewrobustcmd{\mterms}{\cmdkl{\textit{term}}}
\knowledgenewrobustcmd{\withT}[1]{(\T,#1)} 

\knowledgenewrobustcmd{\IQ}{\cmdkl{\color{red}\mathsf{IQ}}} 
\knowledgenewrobustcmd{\AQ}{\cmdkl{\mathsf{AQ}}} 
\knowledgenewrobustcmd{\CQ}{\cmdkl{\mathsf{CQ}}} 
\knowledgenewrobustcmd{\FO}{\cmdkl{\mathsf{FO}}} 
\knowledgenewrobustcmd{\sjfCQ}{\cmdkl{\mathsf{sjf\text{-}CQ}}} 
\knowledgenewrobustcmd{\indepCQ}[1]{\mathsf{\cmdkl{indep\text{-}(}#1,\CQ)}} 
\knowledgenewrobustcmd{\CQneq}{\cmdkl{\mathsf{CQ}^{\neq}}} 
\knowledgenewrobustcmd{\CQeq}{\cmdkl{\mathsf{CQ}^{=}}} 
\knowledgenewrobustcmd{\CQeqneq}{\cmdkl{\mathsf{CQ}^{\neq,=}}} 
\knowledgenewrobustcmd{\UCQ}{\cmdkl{\mathsf{UCQ}}} 
\knowledgenewrobustcmd{\UCQneq}{\cmdkl{\mathsf{UCQ}^{\neq}}} 
\knowledgenewrobustcmd{\CRPQ}{\cmdkl{\mathsf{CRPQ}}} 
\knowledgenewrobustcmd{\UCRPQ}{\cmdkl{\mathsf{UCRPQ}}} 
\knowledgenewrobustcmd{\ACQ}{\cmdkl{\mathsf{ACQ}}} 
\knowledgenewrobustcmd{\sjfACQ}{\cmdkl{\mathsf{sjf\text{-}ACQ}}} 
\knowledgenewrobustcmd{\lincq}{\cmdkl{\mathsf{LinCQ}}}
\knowledgenewrobustcmd{\sjflincq}{\cmdkl{\mathsf{sjf\text{-}LinCQ}}}
\knowledgenewrobustcmd{\sjfstarcq}{\cmdkl{\mathsf{sjf\text{-}StarCQ}}}
\knowledgenewrobustcmd{\interfree}[1][\dlliter, \CQ]{\cmdkl{\mathsf{itf\text{-}}(#1)}} 

\knowledgenewrobustcmd{\Minsups}[1]{\cmdkl{\mathsf{MS}_{#1}}} 

\knowledgenewrobustcmd{\aC}{\cmdkl{C}} 
\knowledgenewrobustcmd{\AG}[1][]{\cmdkl{\A^{G}_{#1}}} 

\knowledgenewrobustcmd{\anon}{\cmdkl{\mathsf{anon}}} 
\knowledgenewrobustcmd{\withanon}[1]{{#1}^{\cmdkl{\circ}}} 
\knowledgenewrobustcmd{\suppfacts}{\cmdkl{\mathcal{SF}}} 
\knowledgenewrobustcmd{\modelsmu}[1][\mu]{\cmdkl{\models}_{#1}} 

\knowledgenewrobustcmd{\countAns}[1]{\cmdkl{\#_{#1}^{\textsf{hom}}}}
\knowledgenewrobustcmd{\countMS}[1]{\cmdkl{\#_{#1}^{\textsf{ms}}}}
\knowledgenewrobustcmd{\countFMS}[1]{\cmdkl{\#_{#1}^{\textsf{fms}}}}
\knowledgenewrobustcmd{\evalCountMS}[1]{\cmdkl{\textsc{eval-}\#_{#1}^{\textsf{ms}}}}
\knowledgenewrobustcmd{\evalCountFMS}[1]{\cmdkl{\textsc{eval-}\#_{#1}^{\textsf{fms}}}}
\knowledgenewrobustcmd{\eval}{\cmdkl{\textsc{eval-}}}

   \knowledgenewrobustcmd{\paramscorefun}[2][]{\cmdkl{\mathbf{\xi}}^{#2}_{#1}} 
   \knowledgenewrobustcmd{\PARAMscorefun}[2][]{\cmdkl{\Xi}^{#2}} 
\knowledgenewrobustcmd{\paramShapley}[2]{\cmdkl{\textup{SVC}^{#2}_{#1}}} 
   \knowledgenewrobustcmd{\dscorefun}[1][]{\cmdkl{\mathbf{\xi}}^{\cmdkl{\textsf{dr}}}_{#1}} 
   \knowledgenewrobustcmd{\Dscorefun}[1][]{\cmdkl{\Xi}^{\cmdkl{\textsf{dr}}}} 
   \knowledgenewrobustcmd{\dShapley}[1]{\cmdkl{\textup{SVC}^{\textsf{dr}}_{#1}}} 
\newcommand{\minsupindex}{\textsf{ms}}
\knowledgenewrobustcmd{\msscorefun}[1][]{\cmdkl{\mathbf{\xi}}^{\cmdkl{\minsupindex}}_{#1}} 
\knowledgenewrobustcmd{\MSscorefun}[1][]{\cmdkl{\Xi}^{\cmdkl{\minsupindex}}} 
\knowledgenewrobustcmd{\msShapley}[1]{\cmdkl{\textup{SVC}^{\minsupindex}_{#1}}} 
   \knowledgenewrobustcmd{\wscorefun}[2][]{\cmdkl{\mathbf{\xi}}^{#2}_{#1}} 
   \knowledgenewrobustcmd{\Wscorefun}[2][]{\cmdkl{\Xi}^{#2}} 
   \knowledgenewrobustcmd{\wShapley}[2]{\cmdkl{\textup{SVC}^{#2}_{#1}}} 
   \knowledgenewrobustcmd{\Sscorefun}[1][]{\cmdkl{\Xi}^{\cmdkl{\textsf{s}}}} 
   \knowledgenewrobustcmd{\SHARPscorefun}[1][]{\cmdkl{\Xi}^{\cmdkl{\textsf{\#}}}} 

\knowledgenewrobustcmd{\MVC}{\cmdkl{\textsc{\#MinVertexCover}}} 

\knowledgenewrobustcmd{\Shapley}[1]{\cmdkl{\textup{\color{red}SVC}_{#1}}} 
\knowledgenewrobustcmd{\Shapleyn}[1]{\cmdkl{\textup{\color{red}SVC}^{\textsf{\textup{n}}}_{#1}}} %
\knowledgenewrobustcmd{\Ptime}{\cmdkl{\mathsf{P}}} 
\knowledgenewrobustcmd{\BPP}{\cmdkl{\mathsf{BPP}}} 
\knowledgenewrobustcmd{\FP}{\cmdkl{\mathsf{FP}}} 
\knowledgenewrobustcmd{\sP}{\cmdkl{\mathsf{\#P}}} 
\knowledgenewrobustcmd{\FPsP}{\cmdkl{\mathsf{FP}^\mathsf{\#P}}} 
\knowledgenewrobustcmd{\NP}{\cmdkl{\mathsf{NP}}} 
\knowledgenewrobustcmd{\sNP}{\cmdkl{\mathsf{\#NP}}} 
\knowledgenewrobustcmd{\FPsNP}{\cmdkl{\mathsf{FP}^{\mathsf{\#NP}}}} 
\knowledgenewrobustcmd{\PH}{\cmdkl{\mathsf{PH}}} 
\knowledgenewrobustcmd{\sPH}{\cmdkl{\mathsf{\#PH}}} 
\knowledgenewrobustcmd{\FPsPH}{\cmdkl{\mathsf{FP}^{\mathsf{\#PH}}}} 

\knowledgenewrobustcmd{\sigmaless}[1]{\cmdkl{\sigma}_{\!\cmdkl{<}#1}}
\knowledgenewrobustcmd{\sigmaleq}[1]{\cmdkl{\sigma}_{\!\cmdkl{\le}#1}}

\knowledgenewrobustcmd{\sjfAUCQ}{\cmdkl{\textsf{sjf-AUCQ}}} 

\knowledgenewrobustcmd{\poneone}[1]{\cmdkl{q_{#1}^{\exists!}}}%
\knowledgenewrobustcmd{\qevalD}[2]{\cmdkl{#1(#2)}}

\knowledgenewrobustcmd{\Auto}[1]{\cmdkl{\textnormal{Aut}}(#1)}
\knowledgenewrobustcmd{\Qneq}{\cmdkl{\mathcal{Q}^{\neq}}}

\tikzset{
	subtree/.pic={
		\coordinate (-a) at (0,0);
		\coordinate (-b) at (-1,-2);
		\coordinate (-c) at (1,-2);
		\coordinate (-d) at (-0.35,-0.35);
		\coordinate (-e) at (0.35,-0.35);
		\coordinate (-h) at (-1.26,-2.16);
		\coordinate (-i) at (1.26,-2.16);
		\coordinate (-j) at (0.35,0.35);
		\coordinate (-k) at (-0.35,0.35);
		\coordinate (-west) at (-0.4,0);
		\coordinate (-east) at (0.4,0);
		
		\draw (-a) -- (-b) -- (-c) -- cycle;
		\draw[fill=white] (-a) circle (0.4);
	}
}

\tikzset{
	subtree_small/.pic={
		\coordinate (-a) at (0,0);
		\coordinate (-b) at (-.75,-1);
		\coordinate (-c) at (.75,-1);
		\coordinate (-d) at (-0.35,-0.35);
		\coordinate (-e) at (0.35,-0.35);
		\coordinate (-h) at (-0.92,-1.06);
		\coordinate (-i) at (0.92,-1.06);
		\coordinate (-j) at (0.35,0.35);
		\coordinate (-k) at (-0.35,0.35);
		\coordinate (-west) at (-0.4,0);
		\coordinate (-east) at (0.4,0);
		
		\draw (-a) -- (-b) -- (-c) -- cycle;
		\draw[fill=white] (-a) circle (0.4);
	}
}

\tikzset{
	subtrinvisible/.pic={
		\coordinate (-a) at (0,0);
		\coordinate (-b) at (-1,-2);
		\coordinate (-c) at (1,-2);
		\coordinate (-d) at (-0.35,-0.35);
		\coordinate (-e) at (0.35,-0.35);
		\coordinate (-h) at (-1.26,-2.16);
		\coordinate (-i) at (1.26,-2.16);
		\coordinate (-j) at (0.35,0.35);
		\coordinate (-k) at (-0.35,0.35);
		\coordinate (-west) at (-0.4,0);
		\coordinate (-east) at (0.4,0);
	}
}

\tikzset{
	graphbox/.pic={
		\node [draw, circle, minimum height=17pt] (-0) at (-1, 0.75) {};
		\node [draw, circle, minimum height=17pt] (-1) at (-1, -0.25) {};
		\node [draw, circle, minimum height=17pt] (-2) at (-1, -1.25) {};
		\node [draw, circle, minimum height=17pt] (-3) at (1, -1.25) {};
		\node [draw, circle, minimum height=17pt] (-4) at (1, -0.25) {};
		\node [draw, circle, minimum height=17pt] (-5) at (1, 0.75) {};
		\node [] (-6) at (0, 1.25) {\Huge $G$};
		\coordinate (-7) at (-1.5, -1.75) {};
		\coordinate (-8) at (1.5, 1.75) {};
		\coordinate (-north) at (0, 1.75);
		\coordinate (-west) at (-1.5, 0);
		\coordinate (-east) at (1.5, 0);
		\coordinate (-south) at (0, -1.75);
		
		\draw (-8) rectangle (-7);
		\draw (-0) to (-5);
		\draw (-1) to (-3);
		\draw (-2) to (-4);
	}
}

\tikzset{
	graphbox-ns/.pic={
		\coordinate (-7) at (-1.5, -1.75) {};
		\coordinate (-8) at (1.5, 1.75) {};
		\draw[fill=white] (-8) rectangle (-7);
		
		\node [draw, circle, minimum height=17pt] (-0) at (-1, 0.75) {};
		\node [draw, circle, minimum height=17pt] (-1) at (-1, -0.25) {};
		\node [draw, circle, minimum height=17pt] (-2) at (-1, -1.25) {};
		\node [draw, circle, minimum height=17pt] (-3) at (1, -1.25) {};
		\node [draw, circle, minimum height=17pt] (-4) at (1, -0.25) {};
		\node [draw, circle, minimum height=17pt] (-5) at (1, 0.75) {};
		\node [] (-6) at (0, 1.25) {\Huge $G$};
		\coordinate (-north) at (0, 1.75);
		\coordinate (-west) at (-1.5, 0);
		\coordinate (-east) at (1.5, 0);
		\coordinate (-south) at (0, -1.75);
		
		\draw[very thick,double distance=3pt] (-north) -- ++(0,0.7);
		\draw[very thick,double distance=3pt,arrows = {-Implies[]}] (-south) -- ++(0,-1.6);
		\draw (-8) rectangle (-7);
		\draw (-0) to (-5);
		\draw (-1) to (-3);
		\draw (-2) to (-4);
	}
}

\tikzset{
	graphbox-ew/.pic={
		\node [draw, circle, minimum height=17pt] (-0) at (-1, 0.75) {};
		\node [draw, circle, minimum height=17pt] (-1) at (-1, -0.25) {};
		\node [draw, circle, minimum height=17pt] (-2) at (-1, -1.25) {};
		\node [draw, circle, minimum height=17pt] (-3) at (1, -1.25) {};
		\node [draw, circle, minimum height=17pt] (-4) at (1, -0.25) {};
		\node [draw, circle, minimum height=17pt] (-5) at (1, 0.75) {};
		\node [] (-6) at (0, 1.25) {\Huge $G$};
		\coordinate (-7) at (-1.5, -1.75) {};
		\coordinate (-8) at (1.5, 1.75) {};
		\coordinate (-north) at (0, 1.75);
		\coordinate (-west) at (-1.5, 0);
		\coordinate (-east) at (1.5, 0);
		\coordinate (-south) at (0, -1.75);
		
		\draw[thick,double distance=3pt] (-west) -- ++(-1,0);
		\draw[thick,double distance=3pt,arrows = {-Implies[]}] (-east) -- ++(1,0);
		\draw (-8) rectangle (-7);
		\draw (-0) to (-5);
		\draw (-1) to (-3);
		\draw (-2) to (-4);
	}
}

\theoremstyle{plain}
\newtheorem{thm}{Theorem}
\newtheoremrep{claim}[thm]{Claim}
\newtheoremrep{corollary}[thm]{Corollary}
\newtheoremrep{conjecture}[thm]{Conjecture}

\newtheorem{example}[thm]{Example}

\newtheoremrep{lemma}[thm]{Lemma}

\newtheoremrep{problem}[thm]{Problem}
\newtheoremrep{proposition}[thm]{Proposition}
\newtheoremrep{theorem}[thm]{Theorem}
\crefname{claim}{Claim}{Claims}
\crefname{conjecture}{Conjecture}{Conjectures}
\crefname{definition}{Definition}{Definitions}
\crefname{example}{Example}{Examples}
\crefname{hypothesis}{Hypothesis}{Hypotheses}
\crefname{observation}{Observation}{Observations} 
\crefname{proposition}{Proposition}{Propositions}
\crefname{theorem}{Theorem}{Theorems}

\theoremstyle{remark}

\setlist{left= 0pt}

\title{Tractable Responsibility Measures for Ontology-Mediated Query Answering}
\author{%
Meghyn Bienvenu
\and
Diego Figueira
\and
Pierre Lafourcade
\\
\affiliations
Univ. Bordeaux, CNRS,  Bordeaux INP, LaBRI, UMR 5800, F-33400, Talence, France\\
\emails{
\{meghyn.bienvenu, diego.figueira, pierre.lafourcade\}@u-bordeaux.fr}
}

\begin{document}

\maketitle

\begin{abstract}

Recent work on quantitative approaches to explaining query answers
employs responsibility measures to assign scores to facts in order to quantify their 
respective contributions to obtaining a given answer. 
In this paper, we study the complexity of computing such responsibility scores in the setting of 
ontology-mediated query answering, 
focusing on a very recently introduced family of Shapley-value-based responsibility measures
defined in terms of weighted sums of minimal supports (WSMS). 
By exploiting results from the database setting, we can show that such measures 
enjoy polynomial data complexity for classes of ontology-mediated queries that are first-order-rewritable, 
whereas the problem becomes "shP"-hard when the ontology language can encode reachability queries
(via axioms like $\exists R. A \ic A$). 
To better understand the tractability frontier, we next explore the combined complexity of WSMS computation. 
We prove that 
intractability applies already to atomic queries if the ontology language supports conjunction, 
as well as to unions of `well-behaved' conjunctive queries, 
even in the absence of an ontology. 
By contrast, our study yields positive results for common DL-Lite dialects: by means of careful analysis, 
we 
identify classes of structurally restricted conjunctive queries (which intuitively disallow undesirable interactions between query atoms)
that admit tractable WSMS computation.  \medskip

\end{abstract}

\noindent
\raisebox{-.4ex}{\HandRight}\ \ This pdf contains internal links: clicking on a "notion@@notice" leads to its \AP ""definition@@notice"".%
\ifsubmission
  A long version of this paper can be found at \url{https://arxiv.org/abs/XXXXXXXXXXXXXX}.
\else
  This is the long version of the KR'25 paper \cite{ourKR25}.
\fi

\section{Introduction}
The question of how to explain query answers has received significant attention in both the database 
and ontology settings. Qualitative notions of explanation, based e.g.\ on minimal supports of fact or proofs, 
have been more extensively explored, particular in the ontology setting, cf.\ \cite{DBLP:conf/otm/BorgidaCR08,DBLP:conf/ruleml/AlrabbaaBKK22,DBLP:journals/jair/BienvenuBG19,DBLP:conf/ijcai/CeylanLMV19,DBLP:conf/ecai/CeylanLMV20}. 
However, there has been recent interest in quantitative notions of explanation based upon \emph{responsibility measures}, which assign scores to the dataset facts
to quantify their respective contributions to obtaining a given answer. 
Prior work on responsibility measures for query answers has predominantly focused on the so-called `drastic Shapley value' 
\cite{livshitsShapleyValueTuples2021,DeutchFKM22ComputingShapley,KhalilK23ShapleyRPQ,KaraOlteanuSuciuShapleyBack,ReshefKL20,ourpods24,karmakarExpectedShapleyLikeScores2024,ourKR24}. 
The drastic Shapley value is defined as the Shapley value of the 0/1 modeling of the (Boolean) query. It was motivated by the appealing theoretical characterization of the Shapley value as a concept in cooperative game theory, which is the only distribution of wealth among players respecting certain guarantees, known as `Shapley axioms' \cite{shapley:book1952}. The choice of modelling the query as a 0/1 cooperative game, however, has not been justified.

Unfortunately, the computation of the drastic Shapley value is generally intractable ($\sP$-hard in data complexity), 
even in the absence of ontologies and for very simple (conjunctive) queries \cite{livshitsShapleyValueTuples2021,ourpods24}.
Furthermore, it has recently been argued in \cite{ourpods25} that: 
    (i) not all Shapley axioms yield desirable properties when translated into the query answering setting, and 
    (ii) the genuinely desirable properties for responsibility measures of query answers do not pinpoint a  single best 
    score function.

In light of this, \cite{ourpods25} has very recently proposed a family of responsibility measures, based on \emph{weighted sums of minimal supports} (\emph{WSMS}), where the score of a fact is defined as a weighted sum of the sizes of the query's minimal supports containing it.
The cited work shows that WSMS satisfy several desirable properties and that they enjoy more favourable computational properties compared to the drastic Shapley value in the database setting. Further, WSMS can also be defined as the Shapley value of suitable cooperative games.

The positive results for WSMS in the database setting motivate us to investigate the complexity of computing WSMS responsibility scores in the more challenging setting of \emph{ontology-mediated query answering} (\emph{OMQA}) \cite{poggiLinkingDataOntologies2008,bienvenuOntologyMediatedQueryAnswering2015,xiaoOntologyBasedDataAccess2018}. For this first study of WSMS
in the ontology setting, we focus on description logic (DL) ontologies \cite{DBLP:books/daglib/0041477},
paying particular attention to DLs of the DL-Lite family \cite{calvaneseetal:dllite}, which are the most commonly adopted in OMQA,
due to their favourable computational properties.
We thus consider \emph{ontology-mediated queries (OMQs)} of the form $(\T, q)$,
where $\T$ is formulated in some DL and $q$ is either a conjunctive query (CQ) or atomic query. 

\subsubsection*{Contributions} 
Our results show that the good computational behaviour of WSMS in the database setting exhibited in \cite{ourpods25} extends to some relevant classes of ontology-mediated queries. 
This is in sharp contrast to the intractability of the drastic Shapley measure considered in the database \cite{livshitsShapleyValueTuples2021,ourpods24} and ontology \cite{ourKR24} settings.
    
More precisely, we observe that 
WSMS computation is tractable in data complexity 
for any UCQ-rewritable OMQ. In particular, this covers the class of OMQs $(\T,q)$ consisting of a DL-Lite ontology and CQ $q$ 
 (\Cref{th:data-ucq}). We show in fact that WSMS computation for such OMQs can be implemented using relational database systems via simple SQL queries (\Cref{thm:rewritting-SQL} \& \Cref{cor:redux-to-SQL}). 
We further define a class of `well-behaved' OMQs composed of DL-Lite ontologies and bounded "treewidth" CQs 
for which we establish tractability also in combined complexity (corollary of \Cref{th:cfdllite}).

We also identify DL constructs that render WSMS computation intractable. 
In particular, we show that the data complexity becomes $\sP$-hard for classes of OMQs 
exhibiting reachability behaviour, e.g.\ admitting axioms \mbox{$\exists R. A \ic A$} (\Cref{cor:reachabilityDL-hard-data}).
Furthermore, 
the presence of concept conjunction, present in lightweight DLs like as $\EL$ and Horn dialects of DL-Lite, 
leads to 
$\sP$-hardness in combined complexity (\Cref{prop:conjinstancehard}).
Furthermore, we show that while UCQ-rewritable OMQs enjoy tractable data complexity,
this is not the case for combined complexity, 
even if the rewriting falls within a very restrictive fragment of UCQs (namely, acyclic and self-join free) (\Cref{prop:sjfUCQhard}).

\subsubsection*{Organization}
After reviewing basic terminology and notation, 
we recall in \Cref{sec:shapley} the different responsibility measures, in particular, the recently introduced WSMS,  
explain how they can be applied in the OMQA setting, and define the WSMS computation task. 
\Cref{sec:data} focuses on (in)tractability results on data complexity. The remaining sections consider combined complexity.
\Cref{sec:combined-atomic} presents (in)tractability results for OMQs based upon atomic queries. 
\Cref{sec:ucq} shows that UCQs are generally intractable for WSMS computation.
Finally, \Cref{sec:cf1} exhibits a condition ensuring polynomial-time tractability.
We finish with some concluding remarks in \Cref{sec:conclusion}.


\section{Preliminaries}
We recall key definitions and notation concerning description logics 
and ontology-mediated query answering. 
\subsubsection*{Description Logic Knowledge Bases}
\AP
A ""description logic"" ("DL") ""knowledge base"" ("KB") $\K =(\A,\T)$ consists of an "ABox" $\A$ and a "TBox" $\T$, 
constructed from mutually disjoint sets $\intro*\cnames$ of ""concept names"" (unary predicates), $\intro*\rnames$ of ""role names"" (binary predicates), 
and $\intro*\inames$ of ""individual names"" (constants). 
\AP An ""inverse role"" has the form $r^-$, with $r \in \rnames$,
and we use $\intro*\NRpm =  \rnames \cup \{r^- \mid r \in \rnames\}$ for the set of ""roles"". 
\AP
The ""ABox"" 
is a finite set 
of ""concept assertions"" of the form $\Amsf(c)$ with $A \in \cnames, c \in \inames$ and 
""role assertions"" 
$\rmsf(b, c)$ with $\rmsf \in \rnames, b,c \in \inames$. 
We use \AP$\intro*\const(\A)$ for the set of "individual names" in $\A$,
and 
we write $R(b, c)$ to mean $r(b,c)$ if $R=r \in \rnames$
and $r(c,b)$ if $R=r^-$.  
\AP
The ""TBox"" (ontology) 
is a finite set of ""axioms""
whose form depends on the DL in question. 

\AP
Many of our results concern lightweight DLs of the ""DL-Lite"" family. 
We shall in particular consider the $\intro*{\dlliter}$ dialect, whose TBox
axioms take the form of \emph{concept inclusions} $B \sqsubseteq C$
and \emph{role inclusions} $R \sqsubseteq S$,
built according to the following grammar
\begin{align*}
B := \Amsf \mid \exists R \quad C:= B \mid \neg B\quad
 S:= R \mid \neg R
\end{align*}
where $\Amsf\in \cnames$ and $R \in \NRpm$. 
\AP
The logic $\intro*{\dllitec}$ is obtained from $\dlliter$ by disallowing role inclusions. 
\AP
Another prominent lightweight DL is $\intro*{\EL}$, whose
TBoxes consist of concept inclusions $D_1 \sqsubseteq D_2$
between $\EL$-concepts 
built as follows: 
$$D:= \top \mid \Amsf \mid D \sqcap D \mid \exists \rmsf.D \qquad \Amsf\in \cnames, \rmsf\in \rnames  $$

\AP
The semantics of DL KBs is defined using 
""interpretations"" $\I=(\Delta^{\I},\cdot^{\I})$, 
 where the ""domain"" $\Delta^{\I}$ is a non-empty set 
 and the interpretation function ${.}^{\I}$ maps each $a \in \inames$ to $a^{\I} \in \Delta^{\I}$, 
 each $\Amsf \in \cnames$ to $\Amsf^{\I} \subseteq \Delta^{\I}$, 
 each $\rmsf \in \rnames$ to $\rmsf^{\I} \subseteq \Delta^{\I} \times \Delta^{\I}$.
 The function  $\cdot^{\I}$ is extended to complex concepts and roles: %
 $\top^{\I}=\Delta^{\I}$,
 $(\exists R)^{\I}=\{d \mid \exists e \in \Delta^{\I}, (d, e) \in R^{\I}\}$,
$(\rmsf^-)^\I=\{(e,d) \mid (d,e)\in \rmsf^\I \}$, $(C \sqcap D)^\I = C^\I \cap D^\I$. 
An interpretation $\I$ \emph{satisfies an assertion} $\Amsf(a)$ (resp.\ $\rmsf(b, c)$) if $b^\I \in A^{\I}$ (resp.\ $(b^\I, c^\I) \in \rmsf^{\I}$). 
$\I$ \emph{satisfies a (concept or role) inclusion} $\Gmsf \sqsubseteq \Hmsf$ if $\Gmsf^{\I} \subseteq \Hmsf^{\I}$. 
\AP
We call  $\I$ a ""model of a KB"" $\K$, denoted $\I\models\K$, if $\I$ satisfies all axioms in $\T$ (written $\I\models\T$) and all assertions in $\A$ (written $\I\models\A$). 
It will also be convenient to introduce notation $\K \models \exists R (a)$ (with $R \in \NRpm$) to indicate that $a^\I \in (\exists R)^\I$ in every model $\I$ of $\K$. 
We use $\intro*{\mods}(\K)$ for the set of models of $\K$. 
\AP
A KB $\K$ is ""consistent"" if it has a model.

\subsubsection*{Databases}
\AP
While our main interest is in DL KBs, we shall import definitions and techniques from the database literature,
so we briefly introduce some notation and terminology for databases. Formally, a 
\AP
(relational) ""database"" $\D$ 
is a finite set of relational ""facts"" $P(\vec{a})$, where $P$ is a relational predicate of arity $k\geq 1$ 
and $\vec{a}$ is a $k$-ary vector of ""constants"" (which we may assume drawn from $\inames$). 
The ""signature"" of $\D$ is the set of predicates that occur in the facts of $\D$, 
and we speak of a binary "signature" if only unary and binary predicates are used. 
We will use $\alpha \in \D$ to indicate that fact $\alpha$ occurs in $\D$. 

Every "database" $\D$ can be equivalently viewed as a finite first-order interpretation $\I_\D$
whose "domain" is the set of constants in $\D$, which interprets every constant from $\D$ as itself, 
and which interprets each predicate $P$ 
as follows: $P^\I = \{\vec{a} \mid P(\vec{a}) \in \D\}$. 
Moreover, it will sometimes prove convenient to treat an ABox $\A$ as a "database"  
and to consider the associated finite interpretation $\I_\A$. 

\subsubsection*{Queries} 
An (abstract, non-numeric) ""query"" of arity $k\geq0$ is a function that 
maps every first-order interpretation $\I$ to a set $q(\I)$ of $k$-tuples of "domain" elements. 
\AP We will mostly work with ""Boolean queries"", of arity 0, for which 
$q(\I)$ can only take two values: $\emptyset$ which we interpret as `false', and $\{()\}$ which we interpret as `true'. 
In the latter case, we write $\I \models q$. 
Queries may also be (and in fact are usually) evaluated over databases, 
in which case $q(\D)$ returns a set of tuples of constants from $\D$ (called ""answers"")
and $\D \models q$ means that $q$ evaluates to `true' in 
the associated interpretation $\I_\D$. 

We consider various concrete classes of  
""first-order queries"" ("FO"$\phantomintro{\FO}$-queries for short), 
which are defined as first-order logic formulas. 
Note that when querying DL KBs, the relational atoms in queries 
will use predicates from $\cnames \cup \rnames$, whereas in the database context, 
the atoms will use the available database predicates. 
\AP The most prominent class of "FO" queries is $\intro*{\CQ}$, the class of ""conjunctive queries"", which are defined by formulas of the form $q(\vec x) = \exists \vec{y}. \alpha_1\land \dots \land\alpha_n$ where the $\alpha_i$ are relational ""atoms"" that can contain "constants" and/or ""variables"" in $\intro*{\vars}(q)\defeq \vec{x}\cup\vec{y}$, where $\vec{x}$ is the vector of "free variables" of the formula.
\AP A "CQ" $q$ can be partitioned into is ""connected components@@q"" which are the inclusion-maximal "connected@@q" subqueries of $q$.
\AP When a "CQ" has no "free variables" it is naturally "Boolean@@q".
\AP We shall also consider the subclass $\intro*{\AQ}$ of "CQs" with a single "atom" called ""atomic queries"" (AQs), 
\AP the class $\intro*{\UCQ}$ of ""unions of conjunctive queries""
defined as finite disjunctions of "CQ"s with the same set of free variables, and the class
\AP
$\intro[\CQneq]{\mathsf{(U)CQ}^{\neq}}\phantomintro*{\UCQneq}$ of "(U)CQ@UCQ"
with inequality atoms. 

We will often treat "CQs" as sets of atoms and use notation like $\alpha \in q$ to indicate that atom $\alpha$ is a conjunct of the CQ~$q$.
It is well known that when $q$ is a constant-free "Boolean@@q" "CQ", $\I \models q$ iff there is a homomorphism $h:q \homto \I$, i.e.\ 
a function $h$ from $\vars(q)$ to $\Delta^\I$ such that $P(\vec{x}) \in q$ implies that $h(\vec{x}) \in P^\I$ (or $P(h(\vec{x})) \in \D$ if we 
evaluate $q$ over a database $\D$). This characterization extends to "Boolean@@q" "CQs" with constants by adding the requirement that 
$h$ 
maps every constant $c$ to $c^\I$ (or to $c$ itself, if we work with databases). 
Note that one can define in the same manner homomorphisms between two queries or between two interpretations. 

\subsubsection*{Ontology-Mediated Query Answering}
We say that a "Boolean query" $q$ is ""entailed from a KB""~$\K$,
written $\K \models q$, if $\I \models q$ for every $\I \in \mods(\K)$.
\AP The ""certain answers"" to a non-"Boolean@@q" $k$-ary query $q(\vec{x})$ "wrt"\ a "KB" $\K = (\A, \T)$ 
are the tuples $\vec{a} \in \const(\A)^k$ %
such that $\K \models q(\vec{a})$, 
with $q(\vec{a})$ the "Boolean query" obtained by substituting $\vec{a}$ for 
$\vec{x}$.

Alternatively, we can 
treat $\T$ and $q$ together as constituting a composite \AP ""ontology-mediated query"" (\reintro{OMQ}) $Q=(\T, q)$, 
in which case we will write $\A \models (\T, q)$ 
to mean $(\A, \T) \models q$.
The notation $(\mathcal{L}, \mathcal{Q)}$ will be used to designate the class of all OMQs $(\T, q)$ consisting of a TBox formulated in the DL $\L$
and a query $q \in \mathcal{Q}$.

A common approach to computing certain answers (or checking query entailment) is to rewrite an "OMQ" into another query that can be directly evaluated using a database system. Formally, we call a query $q^*(\vec{x})$ a \AP ""rewriting of an OMQ"" $(\T,q)$ if for every "ABox" $\A$ and candidate answer $\vec{a}$: 
$$\A \models (\T,q(\vec{a})) \quad \text{ iff }  \quad \A \models q^*(\vec{a})$$
If we restrict the above definition by requiring that $q^*$ belong to a class of queries 
$\C$ ("eg" $\FO$ or $\UCQneq$), we call it a \AP ""$\C$-rewriting@FO-rewriting"".
Several dialects of DL-Lite, including $\dlliter$, are known to guarantee the existence of $\FO$-rewritings, 
meaning that every OMQ from $(\dlliter, \CQ)$ possesses an FO-rewriting. Moreover, such rewritings 
in fact can be expressed in $\UCQ$ (or in $\UCQneq$, if the DL admits functional roles or number restrictions). 

In Horn DLs, like $\EL$ and $\dlliter$, every "consistent" "KB" $\K=(\A, \T)$ admits a \AP ""canonical model"" $\intro*\canmod{\A,\T}$
with the property that for every model $\J$ of $\K$, there is a homomorphism $h: \canmod{\A,\T} \homto \J$ that is the identity on $\const(\A)$. 
Importantly, if $\K=(\A, \T)$ admits a "canonical model" $\canmod{\A,\T}$, then for every 
"CQ" $q$ and candidate answer tuple $\vec{a}$: 
$$\K \models q(\vec{a}) \quad \text{ iff } \quad \canmod{\A,\T} \models q(\vec{a})$$
The precise definition of $\canmod{\A,\T}$ depends on the particular Horn DL.  In the case of \dlliter, 
the "domain" $\Delta^{\canmod{\A,\T}}$ of $\canmod{\A, \T}$ consists of all words $a R_{1} \ldots R_{n}$ ($n \geq 0$) such that $a \in \const(\A)$, $R_{i} \in \NRpm$, and:
\begin{itemize}
\item if $n \geq 1$, then $(\A, \T) \models \exists R_1 (a)$ and there is no $b \in \const(\A)$ such that $(\A, \T) \models R_1 (a,b)$
\item for $1 \leq i < n$, $\mathcal{T} \models \exists R_i^-
  \sqsubseteq \exists R_{i+1}$ and \mbox{$R_i^- \ne R_{i+1}$}. 
  \end{itemize}
Elements in $\Delta^{\canmod{\A,\T}} \setminus \const(\A)$ will be 
called ""anonymous elements"". 
The interpretation function is defined as follows:
\begin{align*}
 a^{\canmod{\T,\A}} =    &   \, a \text{ for all } a \in  \const(\A)  \\
 A^{\canmod{\T,\A}} =   &   \,
 \Delta^{\canmod{\A,\T}} \cap (
 \{ a \in \const(\A) \mid (\A,\T) \models A(a) \} \cup  \\ &
 \{ a R_{1} \ldots R_{n}
\mid n \geq 1 \mbox{ and }
\T \models \exists{R_n^-} \sqsubseteq A \}{\color{green} )}  \\
 r^{\canmod{\T,\A}} =   &  \,
\left(\Delta^{\canmod{\A,\T}}\right)^2 \cap (
 \{ (a,b) \mid r(a,b) \in \A \} \, \cup  \\ & 
  \,   \{ (w_{1},w_{2})
\mid  
w_{2} = w_{1}R'  \text{ and } \T \models R' \sqsubseteq r \} \,\cup   \\
&   \,  \{ (w_{2},w_{1}) 
\mid \
w_{2} = w_{1}R'  \text{ and } \T \models R' \sqsubseteq r^{-} \}{\color{green} )} 
\end{align*}

\paragraph{Complexity}

\AP We assume familiarity with the class $\intro*{\FP}$ of functions that can be computed in deterministic polynomial time, as well as the class $\intro*{\sP}$ of functions defined as counting the
accepting runs of a nondeterministic Turing machine.
\AP All hardness results are defined from ""polynomial-time Turing reductions"": we say that $P_1$ "reduces" to $P_2$ and write $P_1\intro*{\polyrx} P_2$ if the problem $P_1$ can be solved in polynomial time given access to a unit-cost oracle that solves $P_2$.


\section{Responsibility Measures \& Shapley Value}\label{sec:shapley}
In this section, we recall the notion of responsibility measure, which provides 
quantitative explanations of query answers, 
and some concrete Shapley-value-based responsibility measures. 
We also explain and illustrate how these notions, defined for databases, transfer to the OMQA setting. 

\subsection{Responsibility Measures for Query Answers}
Although we shall be interested in employing responsibility measures
to quantify the contribution of facts to obtaining an 
"answer" $\vec{a}$ 
to a query~$q(\vec{x})$,
it will actually be more convenient to consider the equivalent task of quantifying contributions to 
satisfying 
the associated Boolean query 
$q(\vec{a})$ (obtained by instantiating the "free variables" $\vec{x}$ of $q$ with $\vec{a}$). 
For this reason, in the remainder of the paper,  
we shall w.l.o.g.\ 
restrict ourselves to "Boolean queries". 

We shall further focus on "monotone" Boolean queries, defined in the database setting 
as queries $q$ such that $\D_1 \models q \Ra \D_2 \models q$ whenever $\D_1\subseteq \D_2$.
Such queries notably include the class of homomorphism-closed queries, 
which covers most well-known classes of "OMQ"s, as well as $\UCQneq$. 
Note that a natural qualitative approach to explaining why a "monotone" "Boolean query" $q$ holds in a database $\D$ is to consider the set $\intro*{\Minsups{q}}(\D)$ of ""minimal supports"" of $q$ in $\D$, defined as 
inclusion-minimal subsets 
$\D' \subseteq \D$ s.t.\ $\D' \models q$. 

\AP
Our focus will be on providing \emph{quantitative} explanations in the form of 
 ""responsibility measures"", which are functions that assign a score to every "fact" in the data, reflecting their contributions to making the query hold. 
Such measures have been formally defined, in the database setting, 
as ternary functions $\phi$ that take as input a "database" $\D$, a "(Boolean) query" $q$ and a "fact" $\alpha\in \D$, and output a numerical value. As this definition is extremely permissive, \cite[§4.1]{ourpods25}\footnote{The definitions in the cited paper slightly differ from the ones we present here since we choose to assign scores to all facts, whereas in prior work, the database is partitioned into sets of 
\emph{exogenous} and \emph{endogenous} facts, with only endogenous facts assigned scores. Removing this distinction simplifies the technical presentation, while still covering what is arguably the most relevant practical setting (in which all facts are treated as endogenous). 
 }
 identifies a set of desirable properties that $\phi$ ought to satisfy, focusing on the case of (Boolean) "monotone" queries. 
While the formal definitions are rather technical and outside the scope of this paper, these properties intuitively state: 
   \AP""(Sym-db)"" if two "facts" are interchangeable w.r.t.\ the query, they should have equal responsibility;
   \AP""(Null-db)"" if a "fact" $\alpha\in\D$ is ""irrelevant"" in the sense that $S\cup\{\alpha\} \models q$ iff $S\models q$ for all $S\subseteq \D$, then $\phi(\D,q,\alpha) = 0$, otherwise $\phi(\D,q,\alpha) > 0$;
   and
   \AP""(MS1)"" ("resp" ""(MS2)"") all other things being equal, a fact that appears in smaller (resp.\ more) "minimal supports" should have higher responsibility.

The notions of responsibility measures and minimal supports can be straightforwardly translated into the OMQA setting: 
it suffices to  
take the ABox as the database and use an OMQ $(\T,q)$ for the database query. 

\begin{example}
Consider the $\EL$ "KB" $(\A,\T)$ defined in \Cref{fig:ex-kb} %
 and the "CQ" $\mathsf{FishBased}(x)$, which we instantiate with the answer $\{x\mapsto cancalaiseSole\}$ to obtain the Boolean "CQ" $q\defeq \mathsf{FishBased}(cancalaiseSole)$. There are 3 minimal supports for the OMQ $Q\defeq(\T,q)$ in $\A$: $\{f_1,f_2\}$, $\{f_3,f_4,f_5\}$ and $\{f_3,f_6,f_7\}$.
We illustrate how the properties defined above translate in this context:
by "(Sym-db)", we have $\phi(\D,Q,f_1)=\phi(\D,Q,f_2)$ (as $f_1$ and $f_2$ appear in the same minimal supports); 
by "(Null-db)", we have $\phi(\D,Q,f_0)=0$ (as it is "irrelevant");
by "(MS1)", we have $\phi(\D,Q,f_1)>\phi(\D,Q,f_4)$ (as $f_1$ appears in smaller supports);
and by "(MS2)", $\phi(\D,Q,f_3)>\phi(\D,Q,f_4)$ (as $f_3$ appears in more supports). 
\end{example}

\begin{figure}[tb]
\small\centering
\[\hspace{-1mm}\begin{array}{ll}
   \exists \mathsf{hasIng}.\mathsf{FishBased} \ic \mathsf{FishBased}&
   \mathsf{hasGrnsh} \ic \mathsf{hasIng}\\
   \mathsf{Seafood}\ic \mathsf{FishBased}&
   \mathsf{Fish}\ic \mathsf{FishBased}\\[.3em]
\hline\\[-.5em]
\multicolumn{2}{c}{\begin{tikzpicture}
	\coordinate (00) at (-2, 0.7);
	\coordinate (01) at (2, 0.7); 
	\coordinate (02) at (-2, -0.5);
	\coordinate (03) at (2, -0.5);
	\coordinate (04) at (2, 1.2);
	\coordinate (05) at (2, 0);
	\coordinate (06) at (-2, 1.9);
	\coordinate (07) at (2, 1.9); 
	\coordinate (08) at (2, 2.4);
	\begin{pgfonlayer}{nodelayer}
		\node [draw,minimum height=4ex,rounded corners] (0) at (00) {$cancalaiseSole$};
		\node [draw,minimum height=4ex,rounded corners] (1) at (01) {$oyster$};
		\node [draw,minimum height=4ex,rounded corners] (2) at (02) {$cancalaiseGarnish$};
		\node [draw,minimum height=4ex,rounded corners] (3) at (03) {$shrimp$};
		\node [] (4) at (04) {$\mathsf{Seafood}$};
		\node [] at (4.north east) {\small$f_5$};
		\node [] (5) at (05) {$\mathsf{Seafood}$};
		\node [] at (5.north east) {\small$f_7$};
		\node [draw,minimum height=4ex,rounded corners] (6) at (06) {$butter$};
		\node [draw,minimum height=4ex,rounded corners] (7) at (07) {$sole$};
		\node [] (8) at (08) {$\mathsf{Fish}$};
		\node [] at (8.north east) {\small$f_2$};
	\end{pgfonlayer}
	\begin{pgfonlayer}{edgelayer}
		\draw [->, >=stealth] (0) to node[midway, left] {$\mathsf{hasGrnsh}$} node[midway, right] {\small$f_3$} (2);
		\draw [->, >=stealth] (2) to node[midway, above=-.1,sloped] {$\mathsf{hasIngr}$} node[midway,below=-.1,sloped] {\small$f_4$} (1.west);
		\draw [->, >=stealth] (2) to node[midway, above=-.1,sloped] {$\mathsf{hasIngr}$} node[midway,below=-.1,sloped,] {\small$f_6$} (3);
		\draw [->, >=stealth] (0) to node[midway, above=-.1,sloped] {$\mathsf{hasIngr}$} node[midway,below=-.1,sloped] {\small$f_1$} (7.west);
		\draw [->, >=stealth] (0) to node[midway, left] {$\mathsf{hasIngr}$} node[midway, right] {\small$f_0$} (6);
	\end{pgfonlayer}
\end{tikzpicture}}
\end{array}\]
\vspace{-0.8em}
\caption{An example "KB", with data and knowledge about a recipe from \protect\cite{escoffierGuideCulinaireAidememoire1903}. The arrows represent "role assertions" and labels around boxes ("eg" $\mathsf{Fish}$) represent "concept assertions".}
\label{fig:ex-kb}
\end{figure}
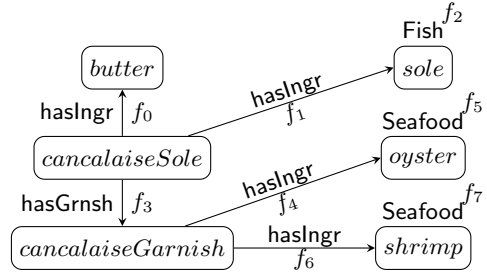

\subsection{Shapley-Based Responsibility Measures}
The "responsibility measures" we consider in this paper are based on the `"Shapley value"'. Originally defined in \cite{shapley:book1952}, it takes as input a cooperative game consisting of a finite set $P$ of players and a ""wealth function"" $\intro*\scorefun: \partsof P \to \lQ$ that assigns a value to each coalition ("ie", set) of players, with $\scorefun(\emptyset)=0$. The \AP""Shapley value"" then assigns to each player $p\in P$ a value $\intro*\Sh(P,\scorefun,p)$ that should be seen as a fair share of the total wealth $\scorefun(P)$ of the game that should be awarded to $p$ based on the respective contributions of all players.
By ``fair share'' we mean that the "Shapley value" is, provably, the only wealth distribution scheme that satisfies a very natural set of axioms \cite[Axioms 1 to 3]{shapley:book1952}.

To obtain a "responsibility measure" from the "Shapley value", one 
needs to 
model the input instance $(\D,q)$ as a "cooperative game" $(P,\xi)$. The set $P$ contains 
the elements that will receive a score, hence it should naturally be the set $\D$ itself. As for the "wealth function", it must assign a numerical score to every database, reflecting in some way the satisfaction of the query. 
Formally, one needs to provide a ""wealth function family"" $\intro*\PARAMscorefun{\star}  \defeq (\paramscorefun[q]{\star})_q$ which associates a "wealth function"~$\intro*\paramscorefun[q]{\star}$ with each query $q$. A "responsibility measure" can be straightforwardly obtained by applying the "Shapley value" to the game $(\D,\paramscorefun[q]{\star})$: 
$$\phi(\D,q,\alpha) \defeq \Sh(\D,\paramscorefun[q]{\star},\alpha)$$
The first "wealth function family" that was considered in the literature is $\intro*\Dscorefun$, defined by: $\intro*\dscorefun[q](\D) \defeq 1$ if $\D \models q$ and $0$ otherwise \cite{livshitsShapleyValueTuples2021}, which gives rise to the \AP""drastic Shapley value"" $\Sh(\D,\dscorefun[q],\alpha)$. In fact, $\Dscorefun$ was until recently the only "wealth function" used to define Shapley-based responsibility measures for Boolean queries.\footnote{As a consequence, the "drastic Shapley value" is simply called `Shapley value' in many papers.} 

Very recently, however, a new family of "responsibility measures" called 
\AP ""weighted sums of minimal supports"" ("WSMSs") has been defined as follows:
\begin{equation}\label{eq:wsms}
   \phi_{\mathsf{wsms}}^w(\+D, q, \alpha) \defeq \sum_{\substack{S\in\Minsups q(\D)\\\alpha\in S}} w(|S|,|\+D|)
\end{equation}
based upon some \AP""weight function"" $w: \Nat \times \Nat \to \lQ$ \cite{ourpods25}. 
It has been shown that all WSMSs can be equivalently defined via the Shapley value: for every "weight function" $w$, there exists a "wealth function family" \AP$\phantomintro{\wscorefun}\intro*\Wscorefun{w} =(\wscorefun[q]{w})_q$ 
such that 
$$\phi_{\mathsf{wsms}}^w(\+D, q, \alpha) = \Sh(\D,\wscorefun[q]{w},\alpha)$$
\cite[Proposition 4.4]{ourpods25}. 
\AP
The "wealth function family" $\phantomintro{\msscorefun}\intro*\MSscorefun \defeq \Wscorefun{w}$ induced by the inverse weight function $w: (n,k) \mapsto \nicefrac{1}{n}$ is of particular interest as its "wealth function" $\msscorefun[q](\D)$ is simply the number  of "minimal supports" for $q$ in $\D$, 
which constitutes a very natural measure of how `robust' the entailment $\D\models q$ is.

It should be noted that the "Shapley values" obtained from $\Dscorefun$ and from $\Wscorefun{w}$ (for any positive and non-decreasing "weight function" $w$) yield "responsibility measures" that satisfy the properties "(Sym-db)"--"(MS2)" \cite[Propositions B.1 and B.2]{ourpods25}. As the following example illustrates, however, these measures do not always coincide, as the properties do not identify a unique `reasonable' responsibility measure.

\begin{example}
Reconsider the $\dllitec$ "KB" $(\A,\T)$ from \Cref{fig:ex-kb} %
and the "CQ" $q\defeq \mathsf{FishBased}(cancalaiseSole)$. Although the properties "(Sym-db)"--"(MS2)" enforce many conditions, they do not restrict the relative values of $f_1$ and~$f_3$. Indeed, we can observe that $\Sh(\D,\dscorefun[q],f_1)>\Sh(\D,\dscorefun[q],f_3)$, but $\Sh(\D,\msscorefun[q],f_1)<\Sh(\D,\msscorefun[q],f_3)$, see \Cref{tb:ex-kb}.
E.g., $\Sh(\D,\msscorefun[q],f_3)= \nicefrac{1}{3} + \nicefrac{1}{3}$ since $f_3$ is in two "minimal supports", both of size 3, and hence each contributing $\nicefrac{1}{3}$.
\end{example}

\begin{table}[tb]
\centering
 \begin{tabular}{|r|c c c c|} 
 \hline
 $\star$ & $f_0$ & $f_1,f_2$ & $f_3$ & $f_4,f_5,f_6,f_7$ \\[0.5ex] 
 \hline\hline&&&&\\[-2ex]
 $\mathsf{dr}$
 & $0$ & $\frac{1224}{5040}$ & $\frac{1056}{5040}$ & $\frac{384}{5040}$\\[0.7ex]
 $\mathsf{ms}$
 & $0$ & $\frac{1}{2}$ & $\frac{2}{3}$ & $\frac{1}{3}$\\[0.5ex] 
 \hline
 \end{tabular}
 \caption{Values of $\Sh(\D,\paramscorefun[q]{\star},\_)$ for the instance in \Cref{fig:ex-kb}
. The facts that are equivalent by "(Sym-db)" are grouped together.}
\label{tb:ex-kb}
\end{table}

Note that while we focus on $\MSscorefun$ as the archetypical "WSMS", it should be observed that the weight function $w$ can be adjusted to suit the needs of particular settings by
giving more or less weight to the size of the "minimal supports" relative to their numbers (intuitively tuning the relative importance of "(MS1)" and "(MS2)"). At the extremes, \cite[§4.4]{ourpods25} introduced two representative "WSMS": \AP$\intro*\Sscorefun$ that always favours appearing in the smallest "minimal supports", and $\intro*\SHARPscorefun$ that always favours the highest total number of "minimal supports".

Following \cite{ourpods25}, for any "wealth function family" $\PARAMscorefun{\star}$ and class $\C$ of queries, we denote by \AP$\phantomintro{\dShapley}\phantomintro{\msShapley}\phantomintro{\wShapley}\intro*\paramShapley{\C}{\star}$ the problem of computing $\Sh(\D,\paramscorefun[q]{\star},\alpha)$ given  any "database" $\D$, fact $\alpha \in \D$, and query $q\in\C$. We also consider the problem $\paramShapley{q}{\star}$ associated with a single fixed query~$q$. 
Our focus in this paper will be on the case $\PARAMscorefun{\star}=\Wscorefun{w}$ for some "weight function" $w$, in particular $\MSscorefun$,
in which case we will speak of ""WSMS computation"". Moreover, we shall study these tasks in the OMQA setting, so 
$\C$ will be a class $(\L,\Q)$ of OMQs, and $q$ will 
be a particular OMQ $Q$.


\section{Data Complexity}\label{sec:data}
We initiate our study of the complexity of "WSMS computation" 
by considering the \emph{data complexity} of the task $\wShapley{(\L,\Q)}{w}$, for different classes $(\L,\Q)$ of "OMQs". 
As usual, data complexity means that complexity is measured only w.r.t.\ the size of the ABox, while the size of the "OMQ" is treated as a 
constant. 
Observe that $\wShapley{(\L,\Q)}{w}$ enjoys $\FP$ data complexity just in the case that 
$\wShapley{Q}{w}$ is in $\FP$ for every "OMQ" $Q \in (\L,\Q)$. Likewise, $\wShapley{(\L,\Q)}{w}$
is $\sP$-hard iff there is some  "OMQ" $Q \in (\L,\Q)$ for which $\wShapley{Q}{w}$ is $\sP$-hard.

Our data complexity results for "OMQs" naturally build upon existing results from the database setting \cite{ourpods25}. 
That work establishes a key lemma that essentially tells us that, assuming the weight function $w$ is "tractable@@wf" and "reversible@@wf", 
the $\wShapley{Q}{w}$ task boils down to counting minimal supports, no more nor less.
\AP
It would be superfluous to give the precise definition of ""tractable@@wf"" and ""reversible@@wf"" weight functions here: simply note that they are minor assumptions that are trivially satisfied by the three explicit instances we consider ($\MSscorefun$, $\Sscorefun$ and $\SHARPscorefun$).\AP
To present this lemma in our setting, we fix a "Boolean@@q" "OMQ" $Q$  ("ie" $Q=(\T,q)$ with $q$ a "Boolean query") and consider two ""numeric queries"" ("ie" "queries" that output a number 
instead of a set of "answers") derived from $Q$: 
\begin{itemize}
   \item $\intro*\countFMS{Q}$ outputs the number $\countFMS{Q}(k,\A)$ of "minimal supports" for $Q$ of size\footnote{FMS stands for \emph{Fixed-size Minimal Supports}.} $k$ in the "ABox" $\A$, for every size $k$
\item $\intro*\countMS{Q}$ outputs the total number of "minimal supports" of $Q$.
\end{itemize}
We denote by $\intro*\evalCountFMS{Q}$ and $\intro*\evalCountMS{Q}$ the problems of computing these numeric queries over any input "ABox".
We now state the key lemma, phrased for OMQs:

\begin{lemma}\label{lem:svc-fms}
   \cite[Lemmas 5.1 and 5.5]{ourpods25}
   For every "reversible@@wf" "tractable weight function" $w$, and "Boolean@@q" "monotone" "OMQ" $Q$, $$\evalCountMS{Q} \polyrx \wShapley{Q}{w} \polyrx \evalCountFMS{Q}.$$
\end{lemma}

\subsection{Tractability Result for Rewritable OMQs}

With this lemma at hand, we can forget about the particular weight function $w$ and 
concentrate on the conceptually simpler task $\evalCountFMS{Q}$ of counting the number of minimal supports, per size. 
This can be achieved in polynomial time (in data complexity) whenever there is a data-independent bound on the 
size of minimal supports, 
since we can then simply iterate over all bounded-size subsets of the ABox. In particular, this holds for OMQs that can be rewritten into a $\UCQneq$, yielding the following tractability result:

\begin{theoremrep}\label{th:data-ucq}
   (Corollary of \cite[Theorem 5.2]{ourpods25})
   $\wShapley{Q}{w}\in\FP$ for every 
   "tractable weight function" $w$ and every  "Boolean@@q" "OMQ" $Q$ that is $\UCQneq$"-rewritable". 
   This implies, in particular, that $\wShapley{(\dlliter,\UCQ)}{w}$ enjoys $\FP$ data complexity. 
\end{theoremrep}
\begin{proof}
   Any "Boolean@@q" "OMQ" $Q$ that is "rewritable" into a $\UCQneq$ must be both "monotone" and `bounded', meaning that the sizes of "minimal supports" for q are bounded by a constant independent of the "database" (namely the maximum size of a $\CQneq$ therein). The result from \cite[Theorem 5.2]{ourpods25} states that in such case $\wShapley{Q}{w}$ is tractable.
\end{proof}

Theorem \ref{th:data-ucq} mentions $(\dlliter,\UCQ)$ to give a concrete example of a tractable OMQ class, 
but naturally the same holds for other prominent DL-Lite dialects (like DL-Lite$_\mathcal{F}$
and DL-Lite$_\mathsf{Horn}$) known to be $\UCQneq$-rewritable. Description logics outside of the DL-Lite family typically do not guarantee the existence of rewritings. However, techniques for identifying $\FO$-rewritable (in fact, $\UCQneq$-rewritable) OMQs for various Horn DLs are known \cite{bienvenuFirstOrderrewritabilityContainment2016a} and have been successfully implemented for OMQs in $(\EL,\CQ)$ \cite{DBLP:conf/semweb/HansenL17}. The preceding tractability result is therefore relevant for a wider range of DLs. 

Beyond providing theoretical results, rewriting is a powerful tool from a practical standpoint and has proven to be a key technique for obtaining efficient implementations of OMQA by 
leveraging existing and highly optimized relational database systems. 
Computing $\wShapley{Q}{w}$ for $\UCQneq$-rewritable OMQs $Q$ is no exception, as it can be done by evaluating simple SQL queries in parallel. 
We phrase the next result for databases, as it is relevant to both the pure database and OMQA settings.  

\newcommand{\aUCQ}{\widetilde{q}}
\begin{theorem}
\label{thm:rewritting-SQL}
For every $\UCQneq$ $\aUCQ$ and $k \in \Nat$, there exists a set $\Qneq$
of $\CQneq$ queries such that,  for every "database" $\D$,
$$\countFMS{\aUCQ}(k,\D) = \sum_{q \in \Qneq} \countAns{q}(\D) \cdot \gamma_q,$$ where \AP$\intro*\countAns{q}(\D)$ is the number of "homomorphisms" from $q$ to $\D$, and $\gamma_q$ is a rational number computable from $q$. 
Further, every $q \in \Qneq$ is of at most quadratic size.
\end{theorem}
\begin{proof}
    Given a query $\aUCQ \in \UCQneq$ and a number $k$, we show how to build the set $\Qneq$ with the desired properties of the statement.
    Let us say that $q'$ is a \AP""reduct"" of $q \in \CQneq$ if $q'$ is the result of collapsing variables of $q$ and removing repeated atoms; in particular, $q'$ is a homomorphic image\footnote{A "homomorphism" $h: q \homto q'$ between $q,q' \in \CQneq$ is defined in the expected way, "ie", 
if $x \neq y \in q$, then $h(x) \neq h(y)$. 
     } of $q$. Similarly, $q'$ is a "reduct" of $\aUCQ$ if it is a "reduct" of a $\CQneq$ therein.
    Consider the set $\R_k$ of all "reducts" $q$ of $\aUCQ$ such that 
    \begin{enumerate}[(a)]
        \item $q$ has exactly $k$ "relational atoms";
        \item there is no "reduct" $q'$ of $\aUCQ$ having strictly less than $k$ "relational atoms" such that $q' \homto q$. \label{eq:minimality-relatoms}
    \end{enumerate}
Let $\Q$ be the result of removing from $\R_k$ all redundant queries. Concretely, we initialize $\Q \defeq \R_k$ and we apply the following \AP""redundancy-removal-rule"" until no more queries can be deleted:
    find a query $q \in \Q$ such that $q' \homto q$ for some distinct $q' \in \Q$ and update $\Q \defeq \Q \setminus \set{q}$.
    Finally, let \AP$\intro*\Qneq$ be the result of adding, for each $q \in \Q$ and distinct $x,y \in \vars(q)$, the 
     atom $x \neq y$ to $q$. In this way, "homomorphisms" from $\Qneq$ queries must necessarily be injective.

    \begin{toappendix}
        \begin{claim}\label{cl:invertible-homomorphisms}
            If $q \in \Qneq$, $M \in \Minsups q (\D)$ and $h: q \homto M$, then $h$ is injective. Further, for any query $p \in \CQneq$ such that $p \homto M$, we have $p \homto q$.
        \end{claim}
        \begin{proof}
            The fact that $h$ is injective follows readily from the "disequality atoms" of $q$ by definition. Hence, it is invertible, and given $g: p \homto M$ we have $(h^{-1} \circ g) : p \homto q$.
        \end{proof}
    \end{toappendix}

    \begin{claimrep}\label{cl:minsup-partition}
    $\set{\Minsups q(\D)}_{q \in \Qneq}$ is a partition of $\set{M \in \Minsups \aUCQ(\D): |M| = k}$.
    \end{claimrep}
    \begin{proof}
        Let us first show 
        \[\set{M \in \Minsups \aUCQ(\D): |M| = k} = \bigcup_{q \in \Qneq} \Minsups q(\D).\] 

        \proofcase{$\subseteq$}
        For the $\subseteq$-containment, first note that if $M$ is a "minimal support" of $\aUCQ$ of size $k$, there must be some $\hat q$ in $\aUCQ$ and some "homomorphism" $h:\hat q \homto M$. Such "homomorphism" $h$ induces a "reduct" $q$ of $\hat q$, in such a way that $q^{\neq} \homto M$, where $q^{\neq}$ is the result of adding $x \neq y$ to $q$ for all $x,y \in \vars(q)$. 
        
        We now claim that either $q^{\neq}$ or a larger query ("ie", a query $p \in \Qneq$ such that $p \homto q^{\neq}$) must be in $\Qneq$. 
        If this is not the case, then it must be that $q^{\neq}$ is not in $\R_k$ because it was discarded by \cref{eq:minimality-relatoms}, which would mean that there is another "reduct" $q'$ of $\aUCQ$ with less than $k$ "relational atoms" such that $q' \homto q^{\neq} \homto M$, implying that there is a "support" strictly included in $M$ and contradicting its "minimality@minimal support".

        \proofcase{$\supseteq$}
        For the $\supseteq$-containment, note that every "minimal support" of $q \in \Qneq$ has size exactly $k$. Take any such "minimal support" $M$ of $q$. By construction, there is some $\hat q$ in $\aUCQ$ such that $\hat q \homto q$ and thus $M$ is also a "support" of $\aUCQ$. By means of contradiction, suppose $M$ is not a "minimal support" of $\aUCQ$, that is, there is $M' \subsetneq M$ and $q'$ in $\aUCQ$ such that $h:q' \homto M'$. Consider the "reduct" $q''$ of $\aUCQ$ induced by $h$, and observe that: 
            (i) $q''$ has strictly less than $k$ "relational atoms" and 
            (ii) $q'' \homto q' \homto M$.
        By \Cref{cl:invertible-homomorphisms} we then have $q'' \homto q$.
        This means that $q$ must not exist in $\Qneq$ since it must have been discarded by \cref{eq:minimality-relatoms}. In view of this contradiction, $M$ is a "minimal support" of $\aUCQ$.

        \proofcase{Disjointness} Let us finally show that parts of the partition are disjoint, that is, $\Minsups q(\D) \cap \Minsups {q'}(\D) = \emptyset$ for all pair of distinct $q,q' \in \Qneq$. By means of contradiction, assume the contrary and let $M \in \Minsups q(\D) \cap \Minsups {q'}(\D)$. 
        Let $h: q \homto M$ and $h':q' \homto M$.
        By \Cref{cl:invertible-homomorphisms} we then have $q \homto q'$.
        By the "redundancy-removal-rule" it cannot be that both distinct queries are in $\Qneq$. This contradiction then shows that there cannot be such $M \in \Minsups q(\D) \cap \Minsups {q'}(\D)$.
    \end{proof}
        An \AP""automorphism"" of $q$ is a "homomorphism" $h: q \homto q$; we denote by $\intro*\Auto{q}$ the set of all "automorphisms" of $q$.
    \begin{claimrep}
    $\countAns{q}(\D) = |\Auto{q}| \cdot |\Minsups q(\D)|$ for all $q \in \Qneq$.
    \end{claimrep}
    \begin{proof}
       In view of the previous \Cref{cl:minsup-partition}, all "minimal supports" of $q \in \Qneq$ have size $k$. Further, each "homomorphism" $q \homto \D$ induces a "minimal support", and the set of "homomorphisms" $\set{ h \mid h: q \homto \D}$ can be partitioned into $\set{h \mid h:q \homto M}_{M \in \Minsups q(\D)}$.
        Hence, it suffices to show, for an arbitrary $q \in \Qneq$ and $M \in \Minsups q(\D)$, that $\countAns{q}(M) = |\Auto{q}|$.

        \proofcase{$\leq$}
        Consider the set $H$ of all (pairwise distinct) "homomorphisms" $h: q \homto M$. 
        Let us fix $g \in H$ and observe that, by \Cref{cl:invertible-homomorphisms}, each $(h^{-1}\circ g) : q \homto q$ is an "automorphism", and further that any two distinct $h_1,h_2 \in H$ give rise to distinct "automorphisms" $h_1^{-1}\circ g$ and $h_2^{-1}\circ g$ of $q$. Hence, $\countAns{q}(\D) \leq |\Auto{q}|$.

        \proofcase{$\geq$} Since $M$ is a "support", let $g : q \homto M$.
        Suppose there are $n$ distinct "automorphisms" $h_1, \dotsc, h_n : q \homto q$. Observe that each $(g \circ h_i): q \homto M$ is a distinct "homomorphism" to $M$, and thus that $\countAns{q}(\D) \geq n= |\Auto{q}|$.
    \end{proof}

\noindent  It is easy to see that all queries $q \in \Qneq$ are of quadratic size. Together with the preceding claims, this 
  shows that $\Qneq$ and numbers $\gamma_q:=\nicefrac{1}{|\Auto{q}|}$ have the required properties. 
\end{proof}
\begin{corollary}\label{cor:redux-to-SQL}
Computing $\wShapley{\aUCQ}{w}$ for a $\UCQneq$ query $\aUCQ$ can be achieved by evaluating simple and short (quadratic) `\textsc{select count(*)}' SQL queries in parallel.
\end{corollary}

\subsection{Intractability Results}
Theorem \ref{th:data-ucq} applies to OMQs which are $\UCQneq$-rewritable (which is equivalent to being $\FO$-rewritable for common DLs),
and we conjecture, in line with \cite[Conjecture 5.7]{ourpods25}, that this is precisely the tractability frontier, i.e.\ if an OMQ $Q$ does not admit a $\UCQneq$-rewriting, then $\msShapley{Q}$ is $\sP$-hard. We recall a related result for regular path queries ("RPQ"s): 

\begin{theorem}\label{thm:apq}
   \cite[Theorem 5.6]{ourpods25}
   Let $w$ be a "reversible@@wf" "tractable weight function", and a \AP""regular path query"" ("RPQ") $q\defeq\+L(c,d)$ ("ie", $q$ tests if there is a path from $c$ to $d$ conforming to a regular language $\+L$). Then $\wShapley{q}{w} \in \FP$ if $\+L$ is finite or $\epsilon\in\+L$ and $c=d$, and $\sP$-hard otherwise.
\end{theorem}

Exploiting the fact that reachability can be expressed using "atomic queries" in DLs such as $\EL$ that admit qualified existential restrictions, we get the following corollary:

\begin{corollaryrep}\label{cor:reachabilityDL-hard-data}
   Let $w$ be a "reversible@@wf" "tractable weight function", 
   and $\L$ be any DL that can express the "axiom" $\exists {r}. {A} \ic {A}$, where ${A}$ is a "concept name" and ${r}$ a "role name". Then, there exists an "OMQ" $Q\in (\L,\AQ)$ such that $\wShapley{Q}{w}$ is $\sP$-hard.
 Thus, $\wShapley{(\L,\Q)}{w}$ is $\sP$-hard in data complexity. 
\end{corollaryrep}
\begin{proof}
Consider the "RPQ" $q\defeq {r}^*(c,d)$, which tests if there is an ${r}$-path from $c$ to $d$. By Theorem \ref{thm:apq}, $\wShapley{q}{w}$ is $\sP$-hard. 
We can use the OMQ $Q=(\{\exists {r}. {A} \ic {A}\},{A}(c)) \in  (\L,\AQ)$ to simulate $q$ (with ${A}$ a fresh concept name). Indeed, it is easily verified that, for every ABox $\A$ consisting of ${r}$-assertions, $\A \models q$ iff $\A \cup \{{A}(d)\} \models Q$. Moreover, $S \in \Minsups{q}(\A)$ iff 
$S \cup \{{A}(d)\} \in \Minsups{Q}(\A \cup \{{A}(d)\})$. 
\end{proof}


\section{Combined Complexity: Atomic OMQs}
\label{sec:combined-atomic}

\begin{toappendix}
   \label{app:combined-aq}
\end{toappendix}

In light of the positive data complexity results for $\UCQneq$-rewritable OMQs,
such as those based upon DL-Lite ontologies, a natural question is whether 
we can achieve tractability even in combined complexity, i.e.\ when also 
taking into account the size of the OMQ. Naturally, this will only be possible if 
the considered class of OMQs admits $\mathsf{PTime}$ query evaluation. 
A natural candidate are atomic OMQs, i.e.\ OMQs $(\T,q)$ where $q \in \AQ$
(we shall consider restricted classes of (U)CQs in Sections \ref{sec:ucq} and \ref{sec:cf1}). 
Note that due to Lemma \ref{lem:svc-fms}, it suffices to consider $\evalCountMS{(\L,\AQ)}$ and 
$\evalCountFMS{(\L,\AQ)}$ to obtain, respectively, lower and upper bounds on $\wShapley{(\L,\AQ)}{w}$,
for any reversible and tractable $w$.

\subsection{DLs with Conjunction}
We first show that "WSMS computation" is intractable in combined complexity for atomic OMQs whenever the considered DL allows for concept conjunction. 
This is the case for the lightweight DL $\EL$ and all of its extensions, 
but also for so-called Horn dialects of DL-Lite (which enjoy tractable data complexity due to Theorem \ref{th:data-ucq}). 

\begin{propositionrep}\label{prop:conjinstancehard}
   Let $\L_\sqcap$ be the DL that only allows for 
   "axioms" of the form ${A}\sqcap {B} \ic {C}$, for ${A},{B},{C} \in \cnames$.
    Then $\evalCountMS{(\L_\sqcap,\AQ)}$  is $\sP$-hard.
\end{propositionrep}
\begin{inlineproof}
We reduce from the problem \AP$\intro*\MVC$ which is to count, given an input graph $G=(V,E)$, the number of inclusion-minimal subsets $S\subseteq V$ such that every edge $e\in E$ has at least one endpoint in $S$. 
This problem has been shown to be $\sP$-hard in \cite[Problem 4]{valiantComplexityEnumerationReliability1979}. 
We prove that $\MVC \polyrx \evalCountMS{(\L_\sqcap,\AQ)}$.
Given a graph $G=(V,E)$, we define the instance:
\begin{align*}
   \T_{G} \defeq &
   \left
   \{A_u \ic B_{(u,v)}, A_v \ic B_{(u,v)} \mid (u,v) \in E\right\} \\
    & \cup \left\{\sqcap_{e\in E} B_e \ic C\right\} \\
  \A_G \defeq & \left\{A_u(c) \mid u\in V\right\} \qquad\qquad q_G \defeq  C(c)
\end{align*}
For simplicity, $\T_G$ uses 
$\sqcap$ of arbitrary arity, but this can be simulated in a standard way (see 
\ifsubmission
   Appendix B%
\else
\Cref{app:combined-aq}%
\fi
) of the full version).
By construction, the "minimal supports" for $(\T_G,q_G)$ in $\A_G$ correspond to the minimal vertex covers of $G$. 
\end{inlineproof}
\begin{toappendix}
   \begin{proof}[Missing details in proof of \Cref{prop:conjinstancehard}]
      In order to turn $\T_G$ into a $\mathsf{MinHorn}$ ontology, we need to emulate the arbitrary conjunction $\T_\sqcap\defeq \{\bigsqcap_{e\in E} B_e \ic C\}$ and the disjunctions $A_u \sqcup A_v \ic B_{(u,v)}$ using $\mathsf{MinHorn}$ "axioms" only. The latter is trivial to do with two "axioms": $\{A_u \sqcap A_u \ic B_{(u,v)},A_v \sqcap A_v \ic B_{(u,v)}\}$. For the former, we write $E\defeq\{e_1\dots e_k\}$ introduce fresh concept names: $\T'_\sqcap\defeq \{ B_{e_1}\sqcap B_{e_2}\ic D_2, \, D_2 \sqcap B_{e_2}\ic D_3, \dots , \, D_{k-1} \sqcap B_{e_k}\ic C\}$. While $\T'_\sqcap$ and $\T_\sqcap$ are not strictly speaking equivalent, they always entail the same assertions with regard to "ABoxes" such as $\A_G$ that don’t contain any of the concept names $D_2 \dots D_{k-1}$.
   \end{proof}
\end{toappendix}

\subsection{DL-Lite Dialects with Singleton Supports}
Many of the more common dialects of "DL-Lite", however, do not allow for conjunction and instead enjoy the following "singleton-support property". 
A DL $\L$ ""has singleton supports"" if for every "atomic@@q" "OMQ" $(\T, q) \in (\L,\AQ)$ and every
ABox $\A$ consistent with $\T$, all sets in  $\Minsups{Q}(\A)$ are singletons. 

\begin{proposition}\AP\label{prop:tractableAQ}
   Let $\L$ be any DL 
   that "has singleton supports" and for which atomic OMQs admit $\mathsf{PTime}$ evaluation in combined complexity. 
  Then $\evalCountFMS{(\L,\AQ)}$ is in $\FP$ for combined complexity. 
  This holds in particular when $\L=\dlliter$. 
\end{proposition}
\begin{proof}
Consider an OMQ $Q=(\T,q)\in(\L,\AQ)$ and an "ABox" $\A$. Since $\L$ "has singleton supports",
we know that every "minimal support" for $Q$ in $\A$ consists of a single assertion from $\A$. We can thus consider the linearly many such singleton sets and use the $\mathsf{PTime}$ evaluation procedure to check which ones entail the given atomic query. The fact that \dlliter\ "has singleton supports" 
is folklore and implicit in existing OMQA algorithms. For example, standard rewriting algorithms will rewrite OMQs from $(\text{\dlliter}, \AQ)$ into unions of AQs, cf.\  \cite{calvaneseetal:dllite}. 
\end{proof}


\section{Digression: Unions of Conjunctive Queries}\label{sec:ucq}
Our next goal will be to extend the tractability result for $(\dlliter, \AQ)$
to cover some suitable subclass of $(\dlliter, \CQ)$, where the user query is a non-atomic CQ. 
Since every $Q \in (\dlliter, \CQ)$ can be rewritten into 
a UCQ $q'$, 
one idea would be to identify conditions on OMQs that 
guarantee that the rewritten query 
belongs to some class $\C$ of UCQs
for which $\evalCountFMS{\+C}$ is known to be tractable. 

Currently, however, it has only been shown that 
$\evalCountFMS{\+C}\in\FP$ for any class $\+C$ of "CQs" 
having bounded ""generalized hypertreewidth"" and bounded `""self-join width""' \cite[Theorem 6.6]{ourpods25}.%
\footnote{Since we do not need the precise definitions, we direct the interested reader to \cite[§6.3]{ourpods25}.}
This covers in particular CQs that are both ""acyclic"" (i.e.\ the undirected graph underlying the query is acyclic)
and ""self-join free"" (i.e.\ no two atoms share the same predicate).
No combined complexity results exist for UCQs, and it is not a priori
clear if the preceding tractability result can be suitably extended to identify a `nice' class of UCQs. 

Our next result shows that positive results for well-behaved CQs do not in fact transfer to their unions. Indeed, $\evalCountMS{}$ is $\sP$-hard already for the restricted class $\AP\intro*\sjfAUCQ$ of "acyclic" "self-join free" "UCQs".

\begin{propositionrep}\label{prop:sjfUCQhard}
    $\evalCountMS{\sjfAUCQ}$ is $\sP$-hard under polynomial-time \AP""$1$-Turing reductions""\footnote{I.e., "Turing reductions" that only allow a single call to an oracle.}, even on binary "signatures" and in the absence of "constants".
\end{propositionrep}
\begin{proofsketch}
    This is an adaptation of the "shP"-hardness proof of \cite[Theorem 6]{PichlerS13} for counting the number of answers of unions of acyclic full conjunctive queries (where `full' means that queries have no existentially quantified variables). The adaptation must address some extra requirements, namely: (i) accounting for counting "minimal supports" instead of answers, (ii) ensuring that the queries are "self-join free", and (iii) working on binary signatures instead of ternary. While this makes the reduction considerably more technical, the underlying idea remains that of Pichler and Skritek.

    The reduction, from the "perfect matching counting" problem, builds a "database" $\+D$ and constant-free queries $q_1,q_2 \in \sjfAUCQ$ such that the number of "perfect matchings" on $G$ is equal to $\countMS{q_1}(\+D) - \countMS{q_2}(\+D)$.
    In fact, $q_1$ is an ("acyclic", "self-join free") "CQ" rather than a "UCQ", and its evaluation is in polynomial time by \cite[Remark 6.5 and paragraph after]{ourpods25}. Hence, this is a "$1$-Turing reduction".
\end{proofsketch}
\begin{proof}
    \newcommand{\xorina}{\textit{OR!-in}^{\text{a}}}%
    \newcommand{\xorinb}{\textit{OR!-in}^{\text{b}}}%
    \newcommand{\xorout}{\textit{OR!-out}}%
    \newcommand{\zero}{\textit{Zero}}%
    \newcommand{\one}{\textit{One}}%
    \newcommand{\rneg}{\textit{Not}}%
    \newcommand{\qval}{\textit{Val}}%
    This is an adaptation of the "shP"-hardness proof of \cite[Theorem 6]{PichlerS13} for counting the number of answers of unions of acyclic full conjunctive queries (where `full' means that queries have no existentially quantified variables). The adaptation must address some extra requirements, namely: (i) accounting for counting "minimal supports" instead of answers, (ii) ensuring that the queries are "self-join free", and (iii) working on binary signatures instead of ternary. While this makes the reduction considerably more technical, the underlying idea remains that of Pichler and Skritek.

    We reduce from the \AP""perfect matching counting"" problem, which is the problem of, given a bipartite graph $G=(V, E)$, counting the number of subsets $M \subseteq E$ (called \AP""perfect matchings"") such that each vertex of $V$ is contained in exactly one edge of $M$. This problem is known to be "shP"-complete \cite[Problem 2]{valiantComplexityEnumerationReliability1979} under polynomial-time "Turing reductions", and was later shown to be hard even under polynomial-time "$1$-Turing reductions" \cite{Zanko91}.

    We will build a "database" $\+D$ and constant-free queries $q_1,q_2 \in \sjfAUCQ$ such that the number of "perfect matchings" on $G$ is equal to $\countMS{q_1}(\+D) - \countMS{q_2}(\+D)$. In fact, $q_1$ is an ("acyclic", "self-join free") "CQ" rather than a "UCQ", and its evaluation is in polynomial time by \cite[Remark 6.5 and paragraph after]{ourpods25}. Hence, this is a "$1$-Turing reduction".
    
    Suppose $V$ is partitioned into $A \cup B$ and assume $A = \set{a_1, \dotsc, a_n}$, $B = \set{b_1, \dotsc, b_n}$, and $E \subseteq A \times B$ (observe that if $|A| \neq |B|$ there are no "perfect matchings").
    For $a_i \in A$, let $B_i = \set{b_j \in B : (a_i,b_j) \in E}$ and we define $A_j$ for $b_j \in B$ analogously.
    The signature for the queries $q_1,q_2$ has binary predicates $\xorina_{i,j}, \xorinb_{i,j},\xorout_{i,j}$, and unary predicates $\zero_{i,j}$, $\one_{i,j}$, $V_{i,j}$ for every $1 \leq i,j \leq n$.
    The database $\+D$ has 5 constants $\set{0,1,00,01,10}$ and consists of the following set of "facts" 
    \begin{align*}
      \xorina_{i,j}(0,00), \xorina_{i,j}(0,01),
      \xorina_{i,j}(1,10), \\ 
      \xorinb_{i,j}(0,00), \xorinb_{i,j}(0,10), 
      \xorinb_{i,j}(1,01), \\
      \xorout_{i,j}(00,0), \xorout_{i,j}(01,1), \xorout_{i,j}(10,1), \\ 
      V_{i,j}(0), V_{i,j}(1), \one_k(1), \zero_k(0)
    \end{align*}
    for every $i,j,k \in [n]$.  
    The $\xorina, \xorinb, \xorout$ "facts" encode part of the `OR' truth table (with two `\emph{in}' arguments and one `\emph{out}' result) in the following sense: The evaluation of the  "CQ" $q_{i,j}^{\mathsf{OR!}}(x,y,z) \defeq \exists w ~ \xorina_{i,j}(x,w) \land \xorinb_{i,j}(y,w) \land \xorout_{i,j}(w,z)$ on $\+D$  is, precisely, the OR truth table without the $(1,1,1)$ line. 
    That is, whenever both arguments $x,y$ are `true', $q_{i,j}^{\mathsf{OR!}}$ doesn't encode the output `true' nor `false' (the constants 1/0) it simply does not hold, which can be thought of as a blocking error state.
   This is done in such a way that the $n$-ary `exists unique' operator $\exists!$ can be defined by chaining the binary operator $\mathsf{OR!}$ (in the same way that chaining the standard $\mathsf{OR}$ yields $\exists$).
    Formally, the evaluation of the "acyclic" and "self-join free" "CQ"\AP\phantomintro{\poneone}:
    \begin{multline*}
      \poneone{i}(x_0, \dotsc,x_m) \defeq \exists z_1,\dotsc,z_m ~ q_{i,1}^{\mathsf{OR!}}(x_0,x_1,z_1) \land {}\\{}\land q_{i,2}^{\mathsf{OR!}}(z_1,x_2,z_2) \land 
      \dotsb \land q_{i,n}^{\mathsf{OR!}}(z_{n-1},x_m,z_m) \land \one_i(z_m)
    \end{multline*}
    for any $m\leq n$
     on $\+D$ returns all $m$-tuples $\bar a \in  \set{0,1}^m$ having exactly one component equal to 1.
  
    Let $\+X$ be the set of variables $x_{i,j}$ such that $(a_i,b_j) \in E$, and let $\+X_{i,*} \defeq \set{x_{i,j} : (a_i,b_j) \in E}$, $\+X_{*,j} \defeq \set{x_{i,j} : (a_i,b_j) \in E}$.
    The valuations of these $\+X$ variables will denote subsets of $E$: if $x_{i,j}$ is assigned $1$ it is in the subset and otherwise (we will make sure that it is assigned $0$) it is not in the subset.
    To be able to recover the assignment form the "minimal supports", we will use the predicates $V_{i,j}$: $\qval \defeq \bigwedge_{x_{i,j} \in \+X} V_{i,j}(x_{i,j})$.

    We now define 
    \[
    q_1 \defeq \qval \land \bigwedge_{i\in [n]} \poneone{i}(\bar x_i).
    \]
    where each $\bar x_i$ is the tuple of variables containing $\+X_{i,*}$.
    Each "minimal support" of $q_1$ on $\+D$ corresponds to a set of edges $E' \subseteq E$ (given by the $V_{i,j}$-facts assigned $1$) such that each vertex from $A$ appears in exactly one edge from $E'$; further the number of "minimal supports" coincides with the number of such set of edges. Observe that $q_1$ is "self-join free" and "acyclic", hence $q_1 \in \sjfAUCQ$.
    However, some of these $E'$ subsets may not be "perfect matchings", since there may be some $b_j$ which does not appear in any edge of $E'$.
    We shall now build a query $q_2$ whose number of "minimal supports" coincides with the number of such $E'$ subsets which are not "perfect matchings".

    For each $j$ let $p_j$ be the query stating that none of the variables from $\+X_{*,j}$ is assigned $1$, that is,
    \begin{align*}
      p_j &\defeq   \zero_1(x_{i_1,j}) \land \dotsb \land \zero_k(x_{i_k,j})
    \end{align*}
    where $\+X_{*,j} = \set{x_{i_1,j}, \dotsc, x_{i_k,j}}$.
    Note that each "minimal support" of $p_j \land q_1$ corresponds to such a set of edges $E' \subseteq E$ which additionally verifies that there are no edges incident to $b_j$. Further, observe that such formula is "self-join free". What is more, the number of "minimal supports" of 
    \[
        q_2 \defeq \bigvee_{j\in [n]} (p_j \land q_1)
    \]
    corresponds to the number of subsets $E' \subseteq E$ such that (i) each vertex from $A$ appears in exactly one edge from $E'$ and (ii) there exists at least one vertex $b_j \in B$ which is not adjacent to any $E'$. Observe that $q_2$ is "self-join free" since each $p_j$ is "self-join free" and there are no predicates in common between $p_j$ and $q_1$, and it is further "acyclic".
  
    It then follows that the number of "minimal supports" of $q_1$ minus the number of "minimal supports" of $q_2$ gives the number of "perfect matchings" of $G$. 
\end{proof}


\section{From Atomic to Conjunctive Queries}
\label{sec:cf1}
We return to the question of how to extend the tractability result for $(\dlliter, \AQ)$ to multiple atoms, covering suitable subclasses of
$(\dlliter, \CQ)$. As seen in Section \ref{sec:ucq}, 
we cannot simply rewrite the OMQ and appeal to tractability results for database queries.
Instead, we shall introduce a class of well-behaved "OMQs", inspired by the class of "self-join free" "CQs". 
We establish tractability for such OMQs by characterizing their minimal supports 
in terms of the minimal supports of the atomic OMQs associated with their atoms.

To simplify the presentation, we assume throughout this section that the TBox is 
formulated in $\dlliter$, but our results also apply to other DL-Lite dialects satisfying 
the conditions of Proposition \ref{prop:tractableAQ}.

\subsection{Interaction-Free OMQs}
Recall that, by \Cref{lem:svc-fms}, it suffices to study $\evalCountFMS{}$ to obtain the tractability of $\msShapley{}$ (and more generally, $\wShapley{}{w}$, for well-behaved 
$w$). The method developed in \cite{ourpods25} for counting "minimal supports" of "CQs" essentially boils down to a reduction to the well-studied problem of counting the homomorphisms of the "CQ" into the database. 
The main issue is that, for arbitrary "CQs", it is possible that several 
homomorphisms map to the same "minimal support". Consider for instance the "CQ" $\exists xy.r(x,y)\land r(y,x)$ and its "minimal support" $\{r(c,d),r(d,c)\}$, which is the image of two homomorphisms: $(x,y)\mapsto (c,d)$ and $(x,y)\mapsto (d,c)$. 
Observe that such situations cannot arise for 
"self-join free" "CQs", as  
each "fact" can only be used to satisfy a single "atom" of the query. 
As a consequence, counting
"minimal supports" reduces to 
counting homomorphisms of the "CQ" into the database, 
which is tractable if there is a bound  on the "generalized hypertreewidth" of the considered "CQs" \cite{PichlerS13}. 
Our aim will be to exhibit a class of OMQs which retains the desirable property that a single fact may not be used to 
satisfy multiple query atoms, thereby allowing us to characterize the minimal supports of such OMQs in terms of the minimal supports of the atomic OMQs of their atoms. 

We now define our tractability criterion. 
\AP
Let $\intro*\anon$ be a special constant not in $\inames$, and for every subset $C\subseteq \inames$ denote by
\AP
$\intro*\withanon{C}$ the set $C\cup\{\anon\}$. Given an ABox $\A$, 
TBox $\T$, "CQ" $q$ and assignment
$\mu : \vars(q) \to \withanon{\const(\A)}$, we write \AP$(\A,\T) \intro*\modelsmu q$ if there exists a "homomorphism" $h: q \to \canmod{\A,\T}$ such that, for every $x\in \vars(q)$, $h(x) = \mu(x)$ if $\mu(x)\in\const(\A)$ and $h(x) \notin \const(\A)$ if $\mu(x)=\anon$. Essentially, the variables assigned to $\anon$ 
map to "anonymous@@c" (non-ABox) elements of the canonical model $\canmod{\A,\T}$.

An "OMQ" $(\T,q) \in (\dlliter, \CQ)$ 
is ""interaction-free"" if, for every "assertion" $f$, atoms $\alpha,\beta$ of $q$ and assignments $\mu_{\alpha} : \vars(\alpha) \to \withanon{\const(f)}$, $\mu_{\beta} : \vars(\beta) \to \withanon{\const(f)}$, we cannot have both $(\{f\},\T)\modelsmu[\mu_{\alpha}] \alpha$ and $(\{f\},\T)\modelsmu[\mu_{\beta}] \beta$, unless $(\alpha,\mu_{\alpha}) = (\beta,\mu_{\beta})$.
 We denote by $\intro*\interfree$
the class of all "interaction-free" "OMQs" in $(\dlliter, \CQ)$.

We believe
 that $\interfree[\dlliter,\CQ]$ is a practically relevant class of OMQs:  8 of the 14 queries from the well-known LUBM benchmark \cite{DBLP:journals/ws/GuoPH05} correspond to "interaction-free" OMQs and the remaining 6 are "interaction-free" after removing obviously redundant atoms. 
In the case where 
 $\T=\emptyset$, the above condition can only be violated by two distinct atoms $\alpha,\beta$ and two assignments $\mu_\alpha:\vars(\alpha)\to \inames, \mu_\beta:\vars(\beta)\to \inames$ (with no $\anon$ in their image) such that $\mu_\alpha(\alpha)=\mu_\beta(\beta)$. As it turns out, this corresponds to saying that $\alpha$ and $\beta$ are \AP`""mergeable""' in the jargon of \cite[§6.3]{ourpods25}.
The notion of "interaction-free" thus generalizes the notion of queries with no "mergeable" atoms (corresponding to "self-join width" 0)
, which includes all "self-join free" "CQs".

\begin{example}
\begin{enumerate*}[(a)]
\item The "OMQ" $(\emptyset,\exists x. r(c,x)\land r(d,x))$,
   with distinct $c,d\in\inames$, is "interaction-free" despite its "self-join", because no fact can satisfy both $r(c,x)$ and $r(d,x)$.
\item The OMQ $(\{\exists r \ic A; \exists r^- \ic A\},\exists x. A(x))$ isn’t "interaction-free" despite the query having a single atom, because the fact $r(c,d)$ satisfies it in two different ways ($x\mapsto c$ and $x\mapsto d$).
\item The OMQ $(\exists x,y. A(x)\land r(x,y),\{A\ic \exists r\})$ isn’t "interaction-free" either, because the fact $A(c)$ would satisfy both atoms thanks to the ontology.
\end{enumerate*}
\end{example}

\subsection{Theorem Statement and Proof Idea}

As previously mentioned, 
we aim to reduce the problem of counting the "minimal supports" of the input OMQ to counting the minimal supports of its component atomic OMQs. Formally, we prove:

\begin{theorem}\label{th:cfdllite}
   Let $\C$ be a subclass of $\interfree[\dlliter,\CQ]$ such that the set of queries $\{q \mid (\T, q) \in \C\}$ has bounded "treewidth".
   Then $\evalCountFMS{\C} \polyrx \evalCountFMS{(\dlliter,\AQ)}$. Further, $\evalCountFMS{\C} \in \FP$.
\end{theorem}

For the proof of \Cref{th:cfdllite}, we focus on the reduction, since the tractability will readily follow from it and \Cref{prop:tractableAQ}. Before diving into the details, we first give the following incomplete but informative formula: 
\begin{equation}\label{eq:intuit}
   \countMS{(\T,q)}(\A) \approx \sum_{\mu : \vars(q) \to \const(\A)} \prod_{\alpha\in q} \countMS{(\T,\mu(\alpha))}(\A)
\end{equation}
Intuitively, this formula enumerates every possible assignment $\mu$, computes the number of "minimal supports" associated with $\mu$ by multiplying the number of possibilities for each atom, 
then sums everything up. This formula does not give the correct result for all OMQs in $\interfree$ (hence the $\approx$), for reasons that will be explained and addressed in \Cref{ssec:anonc}, but it does work in many cases thanks to the following consequence of the absence of "interactions".

\begin{toappendix}
To preface the proofs of \Cref{lem:DqAbar,lem:cfindepsupps}, we consider a few properties of "interaction-free" "OMQs".

On the one hand, from the definition of `"interaction-free"', we know that a given "fact" $f\in\A$ can only be associated to a single pair $(\alpha,\mu_{\alpha})$ with $\alpha$ an atom of $q$ and $\mu_{\alpha} : \vars(\alpha) \to \withanon{\const(\A)}$. From there we can partition $\A$ into $\A\defeq \mathsf{Irr} \uplus\biguplus_{\alpha,\mu_{\alpha}} \suppfacts(\alpha,\mu_{\alpha})$ where we call the elements of
\AP $\intro*\suppfacts(\alpha,\mu_{\alpha})$ the ""supporting facts"" for $(\alpha,\mu_{\alpha})$, that is the facts such that $(\{f\},\T) \modelsmu[\mu_{\alpha}] \alpha$, and where $\mathsf{Irr}$ is a set of "facts" $f$ that are "irrelevant" to $(\T,q)$ because they are not "supporting@@fact" any pair $(\alpha,\mu_\alpha)$. 
To reiterate, the union is disjoint because a given $f\in\A$ can at most be "supporting@@fact" a single pair $(\alpha,\mu_{\alpha})$.

On the other hand, any "minimal support" $S$ for $(\T,q)$ in $\A$ must contain, for every $\alpha\in q$, a fact $f_\alpha$ "supporting@@fact" $(\alpha,\mu_{\alpha})$ for some $\mu_{\alpha} : \vars(\alpha) \to \withanon{\const(\A)}$, with the only condition that the $\mu_{\alpha}$ must all agree on the values of the variables in order to form a complete $\mu : \vars(q)\to\withanon{\const(\A)}$.
Additionally, $S$ cannot contain anything else without breaking minimality.
Overall this means that the set of all "minimal supports" for $(\T,q)$ in $\A$ can be bijectively expressed as
\begin{equation}\label{eq:bij-itf}
   \Minsups{(\T,q)}(\A) \simeq \biguplus_{\mu:\vars(q)\to\withanon{\const(\A)}} \prod_{\alpha\in q} \suppfacts(\alpha,\mu|_{\alpha})
\end{equation}
where $\simeq$ indicates the existence of a bijection between two sets, $\prod$ denotes the Cartesian product of sets and $\mu|_{\alpha}$ the restriction of $\mu$ to $\vars(\alpha)$.
\end{toappendix}

\begin{lemmarep}\label{lem:cfindepsupps}
Let $(\T,q) \in \interfree$, and $\A$ be an "ABox". Then for every assignment $\mu : \vars(q) \to \const(\A)$:
\[
   \countMS{(\T,\mu(q))}(\A) = \prod_{\alpha\in q} \countMS{(\T,\mu(\alpha))}(\A)
\]
\end{lemmarep}
\begin{proof}
We wish to apply \Cref{eq:bij-itf} to express the cardinality of $\Minsups{(\T,\mu(q))}(\A)$ for a specific $\mu$.
To avoid confusion with the $\mu$ in the union, we will change the notations and simply focus on computing $\Minsups{(\T,q_{\mathsf{triv}})}(\A)$ for a query $q_\mathsf{triv}$ with no variable. We then denote by $\epsilon$ the only (trivial) assignment from $\emptyset$ to any set, and apply \Cref{eq:bij-itf} to $q_{\mathsf{triv}}$:
$\Minsups{(\T,q_{\mathsf{triv}})}(\A) \simeq \prod_{\alpha\in q_{\mathsf{triv}}} \suppfacts(\alpha,\epsilon|_{\alpha})$.
Given the fact that the assignment $\epsilon$ does nothing and conflating a singleton with its element, the elements of $\suppfacts(\alpha,\epsilon|_{\alpha})$ are then by definition the "minimal supports" for $(\T,\alpha)$, hence
$\Minsups{(\T,q_{\mathsf{triv}})}(\A) \simeq \prod_{\alpha\in q_{\mathsf{triv}}} \countMS{(\T,\alpha)}(\A)$.
Renaming $q_{\mathsf{triv}}$ back to $\mu(q)$ gives the desired equality.
\end{proof}

Even when \Cref{eq:intuit} holds, it has two issues: 
(a) it does not directly yield a polynomial-time procedure 
as it sums over an exponential number of mappings $\mu$ (this will be addressed in \Cref{ssec:effsum}); and 
(b) 
it computes the value of $\countMS{}$ while we actually need $\countFMS{}$. However, (b) is not actually a problem: as we observed in the proof of \Cref{lem:cfindepsupps}, $\T$ "has singleton supports" and $q$ is "interaction-free", so the "minimal supports" for $(\T,q)$ all have the same size as $q$.

\subsection{Efficient Summation Over Assignments}\label{ssec:effsum}

Observe that in \Cref{eq:intuit} we can ignore all assignments $\mu$ such that $(\A, \T) "\not\models@\modelsmu"_{\mu} q$ since the summand is 0. In other words, \Cref{eq:intuit}
is a sum over all "homomorphisms" $h: q \to \canmod{\A,\T}$ whose image is contained in $\const(\A)$. 
As it turns out, efficient summation over "homomorphisms" has been studied in the context of databases annotated with semirings, or
""weighted databases"". These are defined as
\AP$\D = (\D^\dag,\omega)$, where $\D^\dag$ is a "database" and $\omega: \D^\dag \to \Nat$ assigns, to each "fact", a `weight'.
The weight $q(\D) \in \Nat$ associated to the evaluation of a
"Boolean@@q" "CQ" $q=\exists\vec{x}.\bigwedge_{i=1}^k\alpha_i$ to a "weighted database" $\D$ is
$q(\D)\defeq \sum_{h: q \homto \D^\dag} \prod_{i=1}^k \omega(h(\alpha_i))$.
\Cref{eq:intuit}
can thus be seen as
$q$ applied to the "weighted database"
$\D^{q}_{\A} \defeq (\D^\dag,\set{\beta \mapsto \countMS{(\T,\beta)}(\A)}_{\beta \in \D^\dag})$ for $\D^\dag = \set{\mu(\alpha) \mid \alpha\in q, \mu: \vars(\alpha) \to \const(\A)}$.
Since each atom of $q$ contains at most 2 variables,
$\D^{q}_{\A}$ has at most $|q|\cdot|\const(\A)|^2$ "facts". 
For each "fact" $\beta \in \D^\dag$, the weight $\countMS{(\T,\beta)}(\A)$ can be computed by a call to the $\evalCountMS{(\dlliter,\AQ)}$ oracle.
Overall, $\D^{q}_{\A}$ can be built in polynomial time with the oracle. The last step is to compute
$q(\D^{q}_{\A})$,
which can be done in polynomial time for any class of "CQs" with bounded "treewidth".\footnote{This is a trivial adaptation of the algorithm of \cite[Proposition~3.5]{FlumG04} for counting homomorphisms.
It is also a basic case of the more general tractability results of \cite{KhamisNR16,JoglekarPR16}.}

\subsection{Variables Mapped Outside the ABox} 
\label{ssec:anonc}

As mentioned earlier, \Cref{eq:intuit} is inaccurate with respect to $\dlliter$, as evidenced by the following example.

\begin{example}\label{ex:anonc}
   Take $\T= \{A\ic \exists r; \exists r^- \ic B\}$, $q=\exists x.B(x)$ and $\A=\{A(c)\}$. Then the left-hand side of \Cref{eq:intuit} equals $1$ because $\A$ is a minimal support for $(\T,q)$, but the right-hand side equals 0 because the only possible $\mu$ is $x\mapsto c$, but $(\A,\T) \not \models B(c)$.
\end{example}

The issue highlighted by \Cref{ex:anonc} is that some "minimal supports" may only be witnessed by homomorphisms 
that map some variable to an "anonymous element" of $\canmod{\A,\T}$ rather than an ABox "individual". Such minimal supports will thus  
be missed by \Cref{eq:intuit}. 
However, by exploiting the structure of canonical models in $\dlliter$, we can show that the "interaction-free" condition ensures that all
\AP""shared variables"" 
(i.e.\ variables $x\in\vars(q)$ that appear in multiple atoms of $q$)
are necessarily mapped to ABox individuals: 

\begin{lemma}\label{lem:cfsharvar}
   Let $Q=(\T,q)\in\interfree$ and $x$ be a "shared variable" of $q$. Then, for every "homomorphism" $h: q \homto \canmod{\A,\T}$, 
   we have $h(x)\in \const(\A)$. 
\end{lemma}
\begin{proof}
Assume 
for a contradiction that $h: q \homto \canmod{\A,\T}$ and $h(x)\not \in \const(\A)$ for some shared variable $x$. 
From the definition of $\canmod{\A,\T}$, we know that $h(x) = a R_1 \ldots R_n$ for some $a \in \const(\A)$ and roles $R_i$. 
For convenience we suppose $R_1 \in \rnames$, but the proof is analogous if $R_1$ is an inverse role. 
As $x$ is shared, it must appear in at least two distinct atoms $\alpha_1$ and $\alpha_2$  of $q$.  
If $\alpha_i=A(x)$, then we must have $h(x) \in A^{\canmod{\A,\T}}$. 
If $\alpha_i = s(x,y)$ or $\alpha_i=s(y,x)$, then $h(y)$ must either be equal to the unique `predecessor' of $h(x)$, which is $a R_1 \ldots R_{n-1}$, 
or to some immediate `successor' of $h(x)$, which must have the form $a R_1 \ldots R_n R_{n+1}$ for some role $R_{n+1}$.  This follows from the way roles are interpreted in $\canmod{\A,\T}$. 
Thus, all terms in $\alpha_1$ and $\alpha_2$ must be mapped by $h$ either to $a$ or to some anonymous element with prefix $a R_1$. 

The existence of $a R_1 \ldots R_n$ in the domain 
means there is some assertion $f \in \A$ such that $(\{f\}, \T) \models \exists R_1(a)$ and $(\{f\}, \T) \not \models R_1(a,b)$ for any individual $b$. 
It follows that 
$\canmod{\{f\},\T}$ will also contain $a R_1$, and in fact all of the anonymous elements of $\canmod{\A,\T}$ having prefix $a R_1$, and it will interpret concept and role names on these elements in precisely the same way as in $\canmod{\A,\T}$. Thus, $h$ witnesses the satisfaction of $\alpha_1$ and $\alpha_2$
in $\canmod{\{f\},\T}$. But this means that we can use $h$ to define 
assignments $\mu_1: \vars(\alpha_1) \to \withanon{\const(f)}$, $\mu_{2} : \vars(\alpha_2) \to \withanon{\const(f)}$
such that both $(\{f\},\T)\modelsmu[\mu_1] \alpha_1$ and $(\{f\},\T)\modelsmu[\mu_2] \alpha_2$. 
This is impossible as $\alpha_1 \neq \alpha_2$ and $(\T,q)$ is "interaction-free".
\end{proof}

We 
now have three kinds of query atoms to consider. First, those that contain no "shared variables", 
which can be treated by a separate call to the $\evalCountMS{(\dlliter,\AQ)}$ oracle, due to the following easy lemma:

\begin{lemmarep}\label{lem:cfvarcc}
Let $(\T,q)\in\interfree$ with $q=q_1\land q_2$ such that $\vars(q_1)\cap\vars(q_2) = \emptyset$. Then for any "ABox" $\A$, $\countMS{(\T,q)}(\A)=\countMS{(\T,q_1)}(\A)\times\countMS{(\T,q_2)}(\A)$.
\end{lemmarep}
\begin{proof}
Given that $\vars(q_1)\cap\vars(q_2) = \emptyset$, the union of any respective "minimal supports" for $(\T,q_1)$ and $(\T,q_2)$ is a "support" for $(\T,q)$, and conversely any "support" for $(\T,q)$ must contain a "minimal support" for both $(\T,q_1)$ and $(\T,q_2)$. Finally, since $(\T,q)$ is "interaction-free", any respective "minimal supports" for $(\T,q_1)$ and $(\T,q_2)$ must be disjoint, which overall means the set of "minimal supports" for $(\T,q)$ is exactly the Cartesian product of the "minimal supports" for $(\T,q_1)$ and those for $(\T,q_2)$.
\end{proof}

Next, we have the atoms that contain only "shared variables". By \Cref{lem:cfsharvar}, their variables must 
be mapped to $\const(\A)$, hence such atoms are already accounted for by the "weighted database" $\D^{(\T,q)}_{\A}$ built in \Cref{ssec:effsum}.

The third class of atoms are 
the role atoms which contain one "shared variable" and one "unshared variable". 
These will be addressed by constructing a "weighted database" $\overline{\D}^{(\T,q)}_{\A}$ that extends 
$\D^{(\T,q)}_{\A}$ with some extra facts. 
Consider one such atom $\alpha=R(x,y)\in q$, with $x$ and $y$ being "shared@@var" and "unshared@@var" respectively, and $R\in\NRpm$ (since the "unshared variable" could come first).
Again by \Cref{lem:cfsharvar}, 
we know that $x$ must necessarily be instantiated by an "individual" $c\in \const(\A)$, but
$y$ might be mapped to an "anonymous element".  We thus add to $\overline{\D}^{(\T,q)}_{\A}$ a fact $R(c,c_\alpha)$, with $c_\alpha$ a fresh "individual", for every "individual" $c\in \const(\A)$ such that, for some
$f\in\A$, $(\{f\},\T)\models \exists R(c)$ but $(\{f\},\T)\not\models R(c,d)$ for all $d\in\A$, and set the weight of this $R(c,c_\alpha)$ to be the number of such $\{f\}$.

\begin{lemmarep}\label{lem:DqAbar}
   Let $(\T,q)\in \interfree$ such that $q$ is "connected@@q" and $|q|\ge 2$. Then $q\left(\overline{\D}^{(\T,q)}_{\A}\right) = \countMS{(\T,q)}(\A)$.
\end{lemmarep}
\begin{proof}
Given \Cref{eq:bij-itf}
and the definition of how to evaluate "weighted databases", it only remains to show that, for every possible $(\alpha,\mu_{\alpha})$, we have $\omega(\mu_{\alpha}(\alpha)) = |\suppfacts(\alpha,\mu_{\alpha})|$, where $\omega$ is the weight function of the "weighted database" $\overline{\D}^{(\T,q)}_{\A}$ and where $\mu_{\alpha}(\alpha)$ is interpreted as $R(c,c_{\alpha})$ whenever it should be $R(c,\anon)$.

First consider the case where $\mu_{\alpha}(\alpha)$ is of the form $r(c,d)$ with $c,d\in\const(\A)$.
In that case we have $|\suppfacts(\alpha,\mu_{\alpha})| = \countMS{(\T,\mu_{\alpha}(\alpha))}(\A)$ as shown in the proof of \Cref{lem:cfindepsupps}, which is the precise weight that has been assigned to $\mu_{\alpha}(\alpha)$ in $\D^{(\T,q)}_{\A}$ (see \Cref{ssec:effsum}). 

Now assume that $\mu_{\alpha}(\alpha)$ contains $\anon$. Since $q$ is "connected@@q" with at least 2 "atoms", each of them must contain at least one "shared variable", hence the form $r(\anon,\anon)$ is excluded by \Cref{lem:cfsharvar}.
The only case remaining is $R(c,\anon)$ with $c \in \inames$.
Again the elements of $\suppfacts(\alpha,\mu_{\alpha})$ are some form of "minimal supports", but since $\anon$ isn’t a real constant, $\countMS{}$ cannot be called directly. Instead we have to count all "minimal supports" for $(\T,\exists R(c))$ that do not satisfy any $(\T,R(c,d))$ for $d\in \const(\A)$. Again this is the weight that we assign to $R(c,c_{\alpha})$ in $\overline{\D}^{(\T,q)}_{\A}$.
\end{proof}

\subsection{Putting everything together}
The construction has been presented in a progressive manner for ease of understanding. 
We now 
 recapitulate the argument in a more direct fashion.

\begin{proof}[Proof of \Cref{th:cfdllite}]
   The algorithm goes as follows. For every "connected component@@q" $q_c$ of $q$ with at least 2 atoms we:
   (1) build the "weighted database" $\D^{q_c}_{\A}$ described in \Cref{ssec:effsum} using the $\evalCountMS{(\dlliter,\AQ)}$ oracle;
   (2) extend $\D^{q_c}_{\A}$ into $\overline{\D}^{q_c}_{\A}$ as described in \Cref{ssec:anonc};
   (3) compute $q_c(\overline{\D}^{q_c}_{\A})$ using standard "weighted database" algorithms. 

   By \Cref{lem:DqAbar}, this yields the value of $\countMS{(\T,q_c)}(\A)$.
   The remaining "connected components@@q" consist in a single atom, so the corresponding value can then be obtained by a direct call to the $\evalCountMS{(\dlliter,\AQ)}$ oracle.
   Once all the values are obtained, we finally multiply them all together, which yields the desired $\countMS{(\T,q)}(\A)$ by \Cref{lem:cfvarcc}.

Regarding the consequence that $\evalCountFMS{\C} \in \FP$, this is a direct application of \Cref{prop:tractableAQ}
to the above.
\end{proof}


\section{Conclusion and Future Work}
\label{sec:conclusion}
Our work explores the recently introduced class of Shapley-based responsibility measures, known as WSMS, in the context of ontology-mediated query answering.
Our complexity analysis 
pinpoints sources of intractability but also identifies relevant classes of OMQs for which WSMS computation is tractable in data (and sometimes also combined) complexity and can moreover be computed using standard database systems. It would be interesting in future work to test out the approach in practice
and try to generalize the `interaction-free' condition to identify further tractable cases.

While we focused on DLs, 
many results extend to other ontology formalisms such as existential rules. In particular, the data tractability result extends to UCQ-rewritable rulesets, and the tractability result for atomic queries 
extends to bounded-arity linear existential rules because they satisfy the conditions of \Cref{prop:tractableAQ}. An interesting future step would be to see if a useful notion of `interaction-free' could be defined in order to obtain tractability in combined complexity for linear existential rules with CQs. 

\newpage

\section*{Acknowledgements}
This work was partially supported by ANR AI Chair INTENDED (ANR-19-CHIA-0014).

\bibliographystyle{kr}

\bibliography{full,biblio}

\end{document}